\documentclass[twoside,11pt]{article}
\usepackage{jair, rawfonts}

\ShortHeadings{Potential-Based Intrinsic Motivation}
{Forbes, Villalobos-Arias, Wang, Jhala, \& Roberts}
\firstpageno{1}

\usepackage[natbibapa]{apacite}
\usepackage{balance}
\usepackage{amsthm}
\usepackage{amsmath,amsfonts,amssymb,mathtools}

\newtheorem{thm}{Theorem}

\usepackage{subcaption}
\usepackage{booktabs}

\def\argmax{\operatornamewithlimits{argmax}}

\newtheorem{theorem}{Theorem}[section]
\newtheorem{corollary}{Corollary}[thm]

\newtheorem{assumption}{Assumption}

\begin{document}

\title{Potential-Based Intrinsic Motivation: Preserving Optimality With Complex, Non-Markovian Shaping Rewards}

\author{\name Grant C. Forbes \email grantc4bes@gmail.com  \\
        \AND
       \name Leonardo Villalobos-Arias \email lvillal@ncsu.edu \\
       \AND
       \name Jianxun Wang \email jwang75@ncsu.edu \\
       \AND
       \name Arnav Jhala \email ahjhala@ncsu.edu \\
        \AND
        \name David L. Roberts \email dlrober4@ncsu.edu \\
        }

\maketitle

\begin{abstract}
Recently there has been a proliferation of intrinsic motivation (IM) reward-shaping methods to learn in complex and sparse-reward environments. These methods can often inadvertently change the set of optimal policies in an environment, leading to suboptimal behavior. Previous work on mitigating the risks of reward shaping, particularly through potential-based reward shaping (PBRS), has not been applicable to many IM methods, as they are often complex, trainable functions themselves, and therefore dependent on a wider set of variables than the traditional reward functions that PBRS was developed for. We present an extension to PBRS that we prove preserves the set of optimal policies under a more general set of functions than has been previously proven. We also present {\em Potential-Based Intrinsic Motivation} (PBIM) and {\em Generalized Reward Matching} (GRM), methods for converting IM rewards into a potential-based form that are useable without altering the set of optimal policies. Testing in the MiniGrid DoorKey and Cliff Walking environments, we demonstrate that PBIM and GRM successfully prevent the agent from converging to a suboptimal policy and can speed up training. Additionally, we prove that GRM is sufficiently general as to encompass all potential-based reward shaping functions. This paper expands on previous work introducing the PBIM method, and provides an extension to the more general method of GRM, as well as additional proofs, experimental results, and discussion.
\end{abstract}

\section{Introduction}

An increasing amount of work in reinforcement learning (RL) uses intrinsic reward functions, in addition to environmental rewards, to speed convergence to reasonable policies. This approach is particularly widespread in sparse-reward problems, or those that are exploration-heavy, and has had much success in these domains \citep{bellemare2016unifying, pathak2017curiosity, burda2018exploration}.

However, adding a secondary reward term may lead to changes in the set of optimal policies, with unintended, and potentially adverse, consequences. For example, \citet{burda2018large} show that an intrinsic reward that incentivizes visiting areas of the state space where the agent is less able to predict what will happen can result in the agent becoming ``addicted'' to watching a screen with flashing random images. \citet{amodei2016concrete} further discuss related issues. 

These issues can be mitigated through hyperparameter tuning---i.e., by multiplying the intrinsic rewards by some $\alpha$ and decreasing $\alpha$ until the problematic behavior disappears. \citet{chen2022redeeming} implemented an automated generalization of this method. However, this approach may decrease the utility of intrinsic motivation and is not generally guaranteed to preserve the optimal policy set. As an alternative, we extend the potential-based, policy-preserving reward shaping term of \citet{ng1999policy} to arbitrary reward functions, and show that this preserves the set of optimal policies. We contribute:

1. \textbf{An extension of potential based reward shaping (PBRS) to potential functions of arbitrary variables in episodic environments, and a proof that this extension does not alter the set of optimal policies}. We derive an accompanying boundary condition that serves as a sufficient condition for preserving optimality, and this allows for extending PBRS to reward functions, like intrinsic motivation (IM), which are dependent on a more general set of variables than has been addressed previously. 

2. \textbf{A novel suite of methods for converting any arbitrary reward function into this extended potential form}, maintaining the benefits of that function while mitigating its drawbacks, by guaranteeing that such a shaping reward will not alter the set of optimal policies in the underlying environment, and thus cannot be ``hacked'' by an agent that has converged to an optimal policy. We first introduce our baseline method, PBIM, and then GRM, a more general suite of optimality-preserving conversion methods of which PBIM is a member. We also prove that \textit{any} PBRS-based optimality-preserving function must be writeable as a GRM function, and thus that GRM is sufficiently general so as to encompass all possible PBRS-based reward shaping functions.

3. \textbf{An empirical demonstration that our method is effective} as both a safety measure to prevent hacking of intrinsic rewards and, in certain cases, to speed up training. This includes both an empirical demonstration of the efficacy of PBIM, as well as an empirical demonstration of the efficacy of a subset of other GRM functions, some of which outperform both PBIM and the base intrinsic reward being used in our chosen environments.

This paper expands on work originally presented in \cite{forbes2024potential} and \cite{aaai24}.

\section{Related Work}\label{sec:related_work}

The two fields of study most directly relevant to our work are reward shaping---particularly potential-based reward shaping---and intrinsic motivation. We will review the former of these first, particularly the relevant theoretical extensions, followed by the latter.

\subsection{Reward Shaping, Potential-Based And Otherwise}
There are many problems to which we would like to find a solution via RL for which the actual reward signal we want to maximize is extremely sparse. This sparsity can make it prohibitively difficult for any classic reinforcement learning algorithm to actually learn if simply trained on that reward signal by itself. Thus, this sparse signal is often supplemented by an additional signal, so that the objective becomes to maximize
\begin{equation} \label{value_rs}
    V_{M}^\pi = \displaystyle \mathop{\mathbb{E}}_{a \sim \pi, s_0 \sim S_0, s \sim T} \sum_{t = 0}^{N-1} \gamma^t (R_t + F_t),
\end{equation}
where $F_t$ is the (ideally non-sparse) shaping function. The MDP $M = (S, S_0, A, T, \gamma, R)$ is thus replaced by a different, more tractable MDP $M'= (S, S_0, A, T, \gamma, R+F)$. 

Unfortunately, while reward shaping has had much success in allowing for agents to learn in sparse-reward environments, including SOTA success in complex POMDPs such as DOTA 2 \cite{berner2019dota}, it comes with the risk of altering the optimal policy of the underlying environment. When this happens, an agent learning to optimize Equation \ref{value_rs} may ``hack'' the shaping reward, and optimize for it at the expense of the true reward that we are interested in optimizing for, thus achieving arbitrarily bad performance. Examples of this abound in the literature, from an agent learning to ``short circuit'' a racetrack by repeatedly crashing into obstacles \cite{boathack} to a soccer-playing robot rewarded for touching a ball that learned to stand next to the ball and rapidly vibrate \cite{ng1999policy}. A bicycle-riding robot incentivized to move towards a goal was able to obtain infinite reward by riding in small circles, thus obtaining reward for moving near the goal but never reaching it \cite{randlov1998learning}. More such examples are easily obtainable. 
\cite{skalse2022defining} refers to this phenomenon as ``reward gaming,'' and  gives a thorough theoretical framework in which to analyze it.

\subsubsection{Potential-Based Reward Shaping}
Potential-Based Reward Shaping (PBRS) was coined in \citet{ng1999policy}, which examined infinite-horizon MDPs, and those with a single absorbing state that terminates the episode when reached by the agent. In these environments, they showed that the set of optimal policies is unchanged by the addition of a shaping reward of the form 
\begin{equation}\label{Ng_potential}
    F(s,s') = \gamma \Phi(s') - \Phi(s).
\end{equation}
Here, $\Phi(s')$ is an arbitrary potential function that is based solely on the state of the environment. When a shaping reward of this form is added, then the Q-function of the new MDP $M'$, becomes equal to the old Q-function, plus a potential with no action-dependence:
\begin{equation}\label{ng_reference}
     Q^*_{M'}(s,a) = Q^*_M(s,a) - \Phi(s).
\end{equation}
This is due to the telescoping sum of the $\Phi$ terms cancelling all but the first term in the sequence (a trick that is repeated in many later works in this area, including this one). Because there is no action-dependence here, choosing an action to maximize this quantity will also maximize the original Q-function:
\begin{equation}
    \argmax_a Q^*_{M'}(s,a) =  \argmax_a Q^*_{M}(s,a).
\end{equation}
Thus, the set of optimal policies of the underlying MDP is preserved.

This line of reasoning is worth belaboring because it, or one similar to it, will be the crux of many other works focused on preserving an optimal policy, including this work: optimality is conserved here because the new Q function is provably equal to the old Q function plus some quantity that is constant with respect to $a$, or more generally shares a set of maximum $a$.

\cite{ng1999policy} also suggests a choice of potential $\Phi(s) = V^*(s)$ when this quantity is known, and showed that using this as the potential is a gridworld environment improved the rate at which the agent learned to navigate to a series of goal states.

\citet{wiewiora2003principled} extends this treatment to reward shaping terms of the form
\begin{equation} \label{wiewora_potential}
F(s,a,s',a') = \gamma \Phi(s', a') - \Phi(s,a),
\end{equation} 
allowing for action-depended ``advice'' potentials: they note, however, that in order for the theoretical guarantees of \citet{ng1999policy} to hold, the potential $\Phi(s,a)$ must be added back to the Q-value during policy training.
Also note the dependence of Equation~\ref{wiewora_potential} on $a'$, which is necessary for this formulation to work, but requires a reward of the form $R(s,a,s',a')$, rather than the more standard $R(s,a,s')$.\footnote{\citet{wiewiora2003principled} also include a ``lookback'' formulation, wherein knowledge of $a'$ is not required, but rather $a_{t-1}$. This formulation, while not dependent on $a'$, still takes two action values as arguments, and is thus still nonstandard.} \citet{wiewiora2003potential} showed that this form of reward shaping parameter is theoretically equivalent to a thoughtful initialization of parameters for the policy representation. 

While \cite{ng1999policy}, and other prior work cited in this section, focus on model-free RL methods, \citet{asmuth2008potential} show that the formulation of \cite{ng1999policy} can be extended to model-based approaches, and serve to effectively decrease the number of trials required to learn an environment in that domain.

\citet{devlin2012dynamic} extend the formulation of \cite{ng1999policy} further, showing that any shaping reward of the form 
\begin{equation} \label{devlin_potential}
    F(s,t,s', t') = \gamma \Phi(s',t') - \Phi(s,t),
\end{equation}
where $t$ and $t'$ refer to the corresponding time steps of $s$ and $s'$, would not alter the set of optimal policies in a single-agent problem. They also, expanding on the work in \cite{devlin2011theoretical}, prove that it does not alter the set of Nash equilibria in a multi-agent problem. They also prove that, unlike Equation~\ref{wiewora_potential}, shaping rewards of this form cannot be reduced to parameter initialization using the same method as \citet{wiewiora2003potential}.

\citet{harutyunyan2015expressing} combine the expansions to \citet{ng1999policy} of \citet{wiewiora2003principled} and \citet{devlin2012dynamic} into the form
\begin{equation} \label{harutyanyun_potential}
    F(s,a,t,s',a',t') = \gamma \Phi(s',a',t')-\Phi(s,a,t)
\end{equation}
and demonstrate a method for converting shaping functions of the form $R(s,a,s')$ to this potential-based formulation\footnote{\citet{harutyunyan2015expressing} describe these as ``arbitrary reward functions,'' and this is true insofar as it applies to any reward function $R(s,a,s')$ that exists in a traditionally-defined MDP. This is required by the way the proof assumes convergence to a TD fixed point. This assumption is not guaranteed for rewards that are arbitrary in the wider sense we use, including most forms of IM---these are often not stable across time and thus a time-independent value function cannot be expected to converge.}. This is the first method in the literature that gives a ``plug-and-play'' formulation for converting a reward function into one that is potential-based. However, \citet{behboudian2022policy} prove that this formulation does not actually maintain an optimal policy in a general sense. It does maintain an optimal policy for the first action in a given trajectory, but subsequent actions may in fact be non-optimal with respect to the original reward function. 

To remedy this, \citet{behboudian2022policy} present PIES, an alternative method to preserve optimality. PIES preserves optimality by iteratively decreasing the coefficient of the shaping reward term throughout training until it reaches zero. While this alternative method does indeed preserve optimality, it does so effectively through removing $F$ entirely from a portion of the training, and thus requires a balance between utilizing the benefits of the shaping rewards and mitigating their drawbacks.\footnote{Compare to \citet{chen2022redeeming}, who tune shaping reward coefficients for IM specifically, but not to zero.}

All works cited thus far have focused on the initial domain of \citet{ng1999policy}, which is environments that are either infinite-horizon, or which have a set absorbing state that terminates the episode. This distinction, along with the focus on MDPs, is formally emphasized by \citet{eck2016potential}, who also make a further extension of PBRS to Partially-Observable Markov Decision Processes (POMDPs), and show that it can apply equally well in that setting, by defining a potential-based reward
\begin{equation} \label{belief_potential}
    F(b, b') = \gamma\Phi(b') - \Phi(b)
\end{equation}
that operates over the agent's belief states. 

In response to the shortcoming of PBRS when applied to episodic environments originally noted in \cite{eck2016potential}, \citet{grzes2017reward} extended PBRS to environments which terminate after some final time step $N$. They note that, in such an environment, when adding the shaping reward in Equation~\ref{Ng_potential}, there is an additional term of difference between the episodic returns of $M'$ and $M$:
\begin{equation}\label{Grzes_contribution}
    U_{M'}(\tau_N) = U_M(\tau_N) + \gamma^N\Phi(s_N) - \Phi(s_0),
\end{equation}
where $U_{M}(\tau_N)$ is the cumulative return of an agent on MDP $M$ under the state-action trajectory $\tau_N$. The latter of these terms, present even in the infinite-horizon case \cite{ng1999policy} as in Equation \ref{ng_reference}, has no action dependence, and thus cannot affect the set of optimal policies. The former, however, is implicitly action-dependent through the agent's ability to affect its final state, and thus poses a problem for maintaining the optimality of the learned policy with respect to the reward function of the original MDP. \citet{grzes2017reward} addresses this problem by adjusting the potential added at the end of the episode. The potential function essentially then becomes time-dependent, as defined in Equation~\ref{Ng_potential} at all time steps except the last of an episode, where it is zero. Formally:
\begin{equation} \label{grzes_potential}
    \Phi_n = \begin{cases}  0 &  \text{if } n = N \\  \Phi(s) &  \text{otherwise}.\end{cases}
\end{equation}
This treatment applies not only to fixed values of $N$, but to environments where a number of different states could be terminal states, and termination could happen at differing times: here $N$ is treated as ``the time at which the episode ends,'' and can freely vary from episode to episode. This truncating of the potential in the last time step ensures that the problematic term from Equation \ref{Grzes_contribution} will always equal zero, and thus restores the desired optimality guarantees.

\citet{goyal2019using} extends PBRS to a potential based on an ``action frequency vector,'' which contains information about the agent's trajectory over some slice of time, while preserving optimality.

We are extending the potential-based formulation further to accommodate potentials that are a function of an arbitrary set of variables. Most commonly, this will simply be $\Phi(\tau_{0}, \tau_{1}, ... \tau_{M})$, where $\Phi$ is the shaping potential, $M$ is the total number of episodes during training, and $\tau_m = (s_0, a_0, s_1, a_1, ... a_{N_m-1}, s_{N_m})$ is  the full trajectory of states and actions during episode $m$ of training. This is sufficiently general to accommodate most IM terms, such as ICM \cite{pathak2017curiosity} or RND \cite{burda2018exploration}. However, to emphasize a generality that could in principle extend beyond this, rather than writing the dependence explicitly (as in, i.e., the $\Phi(s,a,s',a')$ of \citet{wiewiora2003principled}), we will write either $F_n$ or $\Phi_n$ for simplicity of notation and to emphasize that we are dealing with arbitrary variable-dependence.

\subsection{Intrinsic Motivation}

Intrinsic Motivation (IM) is a paradigm in reinforcement learning that seeks to augment ``extrinsic'' sparse rewards with some automated ``intrinsic'' reward function. These reward functions tend to be relatively environment-agnostic and generalizable, often with emphasis on mathematically capturing some analogue for human motivation, such as a desire for ``exploration'' or ``empowerment.''

IM has proven increasingly useful for complex or sparse-rewards environments in recent years. However, the actual reward shaping terms used in the IM literature lie almost universally outside of the traditional MDP and POMDP frameworks, as they cannot be written as a function of a single state transition, $R(s,a,s')$.

A large portion of IM literature is focused on incentivizing exploration, particularly in sparse-reward environments. Simple versions exist, such as incentivizing taking actions that have not been taken recently \cite{sutton1990integrated} or keeping a tabular list of how often each state has been explored, and rewarding less-visited ones \cite{strehl2008analysis}. Recently, more complex exploration rewards have been developed. Tabular methods have been extended to larger, more complex environments through ``pseudo-counts'' \cite{bellemare2016unifying,ostrovski2017count}, which `count' the number of similar states visited in a learned representation of the (potentially continuous) state space. Curiosity-based methods like Intrinsic Curiosity Module (ICM) \cite{pathak2017curiosity} reward agents for ``surprising'' (maximizing the error rate of) an auxiliary network that is itself trained to predict the environment state dynamics. Random Network Distillation (RND) \cite{burda2018exploration}, similarly, rewards agents for fooling a predictor in a random feature space, rather than one specifically tailored to the environment dynamics, and finds that this is surprisingly very competitive with ICM, despite requiring much fewer computations.
Another common IM method relies on ``empowerment'' \cite{mohamed2015variational}, which is a mutual information metric between the agent's actions and future states. \citet{raileanu2020ride} aims to maximize the impact of an agent's actions on a learned state representation. Empowerment-based IM terms rely on the intuition that in general, across all possible environment and reward function permutations, policies that ``keep their options open,'' so to speak, will obtain higher average returns than those which lead into areas of the state space from which the agent can't return. A theoretical analysis and proof of this concept can be found in \cite{turner2019optimal}. Note that, while the methods we present in this paper are able to accommodate most IM terms, they do require an assumption (Assumption \ref{not_empowerment}) that the intrinsic rewards in question be ``future-agnostic,'' and thus cannot be guaranteed to preserve the optimal policy when used with most empowerment-based IM terms.

All examples thus far have used IM as a method for supplementing (usually sparse) base extrinsic rewards. Recently however, IM without the base reward, either to learn skills to be applied later \cite{eysenbach2018diversity} or to replace external rewards entirely \cite{burda2018large}, has gained attention.
While our work with IM is concerned primarily with situations in which there is a precisely known, accessible extrinsic reward function, the surprising efficacy of these externally reward-less methods in some environments deserves closer theoretical examination.

IM is not the only method for dealing with sparse-reward environments. Other work has leveraged expert demonstrations \cite{vecerik2017leveraging}, including leveraging classic PBRS to convert expert demonstrations to a reward function which theoretically guarantees convergence to an optimal policy even if the demonstration itself is suboptimal \cite{brys2015reinforcement}. Another successful mitigation tool in prior work is hindsight experience replay \cite{NIPS2017_453fadbd}, which trains an agent in ``hindsight'' by taking into account counterfactually possible goal states.

There has been some prior work on the risks of IM. Examples include the ``noisy TV'' problem, where an agent with an exploration term advising it to seek novelty can get distracted from a base task by some particularly stochastic object in its environment \cite{burda2018large}. While \cite{burda2018large} sets up experiments specifically to demonstrate the noisy TV problem by literally placing a noisy TV into a simulated maze, it is a problem that is naturally seen in applications of RL (see, for example \citep{pokemon}, wherein the author encounters a version of this problem while implementing an IM term to train an agent to play Pokemon, a game whose bush sprites are naturally ``noisy'').There is also a tendency of other exploration terms less susceptible to the noisy TV problem, such as RND, still causing agents to become noticeably ``risk seeking'' once they've exhausted all easy-to-obtain intrinsic rewards \citep{burda2018exploration}. There is a large body of theoretical work in this area \cite{amodei2016concrete}, but empirical study remains sparse. We hope that our work here can assist empirical research in this area.

There has been some other work in the area of mitigating adverse effects of IM terms, coming mostly in the form of hyperparameter tuning. \citet{chen2022redeeming} utilize a clever method of automatically tuning up exploration coefficients in exploration-heavy environments and tuning them down where IM is less beneficial. Our methods differ from this in two key ways. Firstly, and most importantly, they deliver vital theoretical guarantees that the set of optimal policies will remain unchanged, and thus that any convergence guarantees apply within the new environment. Secondly, while \cite{chen2022redeeming} requires additional hyperparameters, network architecture, and optimization steps beyond that for the combined loss function, our methods require virtually no additional computational overhead, and address the problem solely by adjusting the reward shaping term to one that guarantees an unchanged set of optimal policies.

\section{Extending Potential-Based Reward Shaping to Functions of Arbitrary Variables} \label{extending_sec}

In an episodic environment, we normally want to choose a policy $\pi$ so as to optimize the value function
\begin{equation} \label{V}
    V_{M}^\pi = \displaystyle \mathop{\mathbb{E}}_{a \sim \pi, s \sim T, R_n \sim R} U_M^\pi.
\end{equation}
Here $U_M^\pi$ is the cumulative discounted return:
\begin{equation} \label{U}
     U_M^\pi = \sum_{n = 0}^{N-1} \gamma^n R_n,
\end{equation}
where the rewards $R_n$ are sampled from acting under policy $\pi$ according to the transition dynamics and reward function of environment $M$. Note that we are here considering the general case where the reward function itself need not be deterministic. We also want to define the discounted future return at some arbitrary time step $t$: 
\begin{equation} \label{U_t}
     U_{M,t}^\pi = \sum_{n = t}^{N-1} \gamma^{n-t} R_n,
\end{equation}
the expectation of which is $V_{M,t}^\pi$. Given this, an optimal policy under $R$ for environment $M$ at time $t$ will satisfy
\begin{equation}
    \pi_{M}^* = \argmax_\pi( V_{M,t}^\pi).
\end{equation}
This optimal policy $\pi_M^*$ will also satisfy
\begin{equation}\label{Q_argmax}
     \pi_{M}^*(s) = \argmax_{a_t}( Q_{M,t}^*)
\end{equation}
where
\begin{equation}\label{Q}
    Q_{M,t}^\pi = R_t + V_{M,t+1}^\pi,
\end{equation}
and $Q_{M}^*$ is taken to be $Q_{M}^\pi$ of the optimal policy $\pi = \pi^*_M$. 
If we now define a new environment $M'$ equivalent to $M$ but with the addition of a shaping reward
\begin{equation} \label{shaping_reward}
    F_t = \gamma \Phi_{t+1} - \Phi_t,
\end{equation}
then we can calculate the return for a trajectory in $M'$
\begin{align} \label{U_M_prime}
    {U_{M',t}^\pi} &{= \sum_{n=t}^{N-1} \gamma^{n-t} (R_n + F_n)} \\
    \label{U_M_prime2}
         &{= \sum_{n = t}^{N-1} \gamma^{n-t} (R_n + \gamma \Phi_{n+1} - \Phi_n).}
\end{align}
In order to prove that adding a shaping reward of the form in Equation \ref{shaping_reward} does not alter the set of optimal policies of the underlying environment, it is sufficient to prove that, at every state and timestep, choosing $a$ to optimize $Q_{M',t}^*$ will necessarily optimize $Q_{M,t}^*$ as well, and vice versa. To investigate the conditions under which this relation will hold, we can reduce Equation~\ref{U_M_prime2} to
\begin{align} \label{new_potential}
{U_{M',t}^\pi =} &{\sum_{n=t}^{N-1} \gamma^{n-t}R_n + \sum_{n=t}^{N-1} \gamma^{n-t}(\gamma\Phi_{n+1} - \Phi_n) }\\
  {=}& { U_{M,t}^\pi + \gamma \Phi_{t+1} - \Phi_t + \gamma^2 \Phi_{t+2} - \gamma \Phi_{t+1}  + }\\
  &{\gamma^3\Phi_{t+3} - \gamma^2 \Phi_{t+2} + \cdots + \gamma^{N-t}\Phi_N - \gamma^{N-(t+1)}\Phi_{N-1} }\\
  {=}&{ U_{M,t}^\pi + \gamma^{N-t}\Phi_N - \Phi_t.}
\end{align}
This is essentially the derivation for Equation~\ref{Grzes_contribution} by \cite{grzes2017reward} with a potentially non-Markovian $\Phi_t$, and generalized to apply to all time steps, rather than just $t=0$. Through an application of Equation \ref{V} and Equation \ref{Q}, this becomes
\begin{equation}
    {
    Q_{M',t}^\pi = Q_{M,t}^\pi + \mathop{\mathbb{E}}_{a \sim \pi, s \sim T, R_n \sim R} \left( \gamma^{N-t}\Phi_N - \Phi_t \right).
    }
\end{equation}
Here we see that $Q_{M',t}^\pi$ differs in expectation from $Q_{M,t}^\pi$ by two terms. If these terms' sum is constant with respect to $a_t$, then Equation \ref{Q_argmax} can be applied to show the equivalence of optimal policies between these two environments. This gives us the condition
\begin{equation}\label{boundary_condition}
     \mathop{\mathbb{E}}_{a \sim \pi, s \sim T, R_n \sim R} \left( \gamma^{N-t}\Phi_N - \Phi_t \right) = \Phi'_t \qquad \forall t \in (0,1,...N-1),
\end{equation}
where $\Phi'_t$ is some arbitrary function that is constant with respect to action $a_{t}$. From here, we can state Theorem \ref{extending_thm}:
\begin{thm}[Sufficient Condition For Optimality]\label{extending_thm}
    The addition of a shaping reward $F_t = \gamma \Phi_{t+1} - \Phi_t$ leaves the set of optimal policies unchanged if Equation \ref{boundary_condition} holds.
\end{thm}
\begin{proof}
Given Equation \ref{boundary_condition}, then $\forall t \in (0,1,...,N-1)$,
\begin{align} \label{policy_set_equivalence}
    {\pi_{M'}^*(s)} &{= \argmax_{a_t}( Q_{M',t}^*) }\\
    &{= \argmax_{a_t} (Q_{M,t}^* + \displaystyle \mathop{\mathbb{E}} \left( \gamma^{N-t}\Phi_N - \Phi_t \right)) }\\
    &{= \argmax_{a_t} (Q_{M,t}^* + \Phi'_t)}\label{prephi} \\
                    &{= \argmax_{a_t} (Q_{M,t}^*) = \pi_M^*(s).}\label{postphi}
\end{align}

Note the step between Equations \ref{prephi} and \ref{postphi}: here we are relying on the $a_t$-independence of $\Phi'_t$ to ensure it doesn't affect the $\argmax_{a_t}$ term.\footnote{This step is similar to a step in the central proof of \cite{ng1999policy}.} This is equivalent to stating the set of optimal policies is unchanged by the shaping reward. 
\end{proof}
It is worthwhile to briefly examine what prior work has done to preserve the condition in Equation \ref{boundary_condition}, in order to emphasize that this is the most general treatment of this problem to date, and to situate it within prior literature. In a non-episodic setting, the $\gamma^{N-t}\Phi_N$ term either drops out (in the infinite-horizon setting) or is definitionally independent of $a_t$ (in the setting with a set absorbing state). Thus, much prior work in this area has focused on solely the $-\Phi_t$ term. This has been dealt with by either restricting the potential to be independent of $a_t$ \cite{ng1999policy, devlin2012dynamic}, restricting it to be independent of $a_t$ in the limit as training continues \cite{behboudian2022policy}, or subtracting this potential where appropriate to accommodate its $a_t$-dependence \cite{wiewiora2003principled}. 

All of these methods' restrictions to $\Phi$ can be viewed as subsets of the general class of shaping functions that satisfy Equation \ref{boundary_condition}. Similarly, prior work in episodic PBRS has restricted itself to $-\Phi_t$ terms that are independent of $a_t$, and thus has dealt with the $\gamma^{N-t}\Phi_N$ term by setting $\Phi_N = 0$ \cite{grzes2017reward}. Again, while this is a valid subset of the larger solution space for Equation \ref{boundary_condition}, it excludes an important set of solutions in which each of these terms, while individually $a_t$-dependent, have this dependence cancel out when they are summed together. As we will see, these solutions have incredible potential applications for novel shaping functions, particularly as a method to incorporate IM methods without changing the optimal policy of the underlying environment.
\section{Converting Intrinsic Motivation
to Potential-Based Reward Functions}\label{converting}

All of the IM examples cited in Section \ref{sec:related_work} can change the set of optimal policies, with possibly adverse effects. Thus mitigating these effects, and using IM while guaranteeing the set of optimal policies isn't altered, is highly desirable. We present a practical and straightforward way to convert most IM rewards to a form that is guaranteed not to alter the set of optimal policies. More formally, we present a method that guarantees not to alter the set of optimal policies for an IM whose terms do not depend on the future actions of the agent. We call this approach Potential Based Intrinsic Motivation~(PBIM).

The trick is to realize that, in all time steps but the last, any arbitrary reward function (including IM) \emph{is already a difference of a potential function in the proper form} due to the recursive relation between rewards and their respective cumulative returns. If we define $F_t$ to be an arbitrary intrinsic reward at time step $t$, and $U_t^\pi$ to be the cumulative discounted intrinsic reward sampled from following policy $\pi$ at time step $t$\footnote{Here, we drop the $M$ subscript for simplicity. Note also that in this section, $U_t^\pi$ exclusively denotes the \emph{intrinsic} discounted return, rather than the sum of intrinsic and extrinsic returns.}, then we can rewrite the standard recursive relation between them as
\begin{equation} \label{rewritten_Bellman}
F_t = U_t^\pi - \gamma U_{t+1}^\pi.
\end{equation}

This is conveniently similar to the necessary potential formulation in Equation~\ref{shaping_reward}. In fact, if we choose $\Phi_t = -U_t^\pi$, these equations become identical. 

Choosing this form for $\Phi_t$ may initially appear untenable in a wide variety of environments, as it seems to imply that we will need to know $U_t$ before beginning training. This would presuppose a level of knowledge about the environment and future trajectory of the agent that is unrealistic. However, this is not the case: we don't have to actually \textit{know} $U_t$ in order to set it equal to $\Phi_t$. We only have to know $F_t$, as it is already a difference of the requisite potential in Equation \ref{rewritten_Bellman}, even if we don't know that potential itself.

If we thus choose $\Phi_t = -U_t^\pi$ and implement an IM term normally, we can investigate under what conditions the optimal policy set is preserved by examining Equation \ref{boundary_condition}. It becomes
\begin{equation}\label{potential_boundary}
     \mathop{\mathbb{E}}_{a \sim \pi, s \sim T, R_n \sim R} \left( U_t^\pi - \gamma^{N-t}U_N^\pi\right) = \Phi'_t \qquad \forall t \in (0,1,...N-1).
\end{equation}
This condition is not satisfied by default for most IM terms, as $U_t^\pi$ will be action-dependent in most interesting environments. It is also unnecessarily complicated in this formulation, as $U_N=0$ for any environment in an episodic setting. These observations motivate the choice of potential
\begin{equation} \label{initial_potential}
    \Phi_t = \begin{cases} -U_0^\pi/\gamma^N, &  \text{if } t = N \\  -U_t^\pi, &  \text{if }  t \neq N, \end{cases}
\end{equation}
which is similar to choosing $\Phi_t = -U_t^\pi$, but with the crucial exception that $\Phi_N =-\frac{U_0^\pi}{\gamma^N}$, by virtue of $N$ being the last time step in the episode. Our choice of potential here for the $t=N$ case is motivated by setting $\Phi'_t$ in Equation \ref{potential_boundary} to 0 for the $t=0$ case, and solving for $\Phi_N$.

Thus, if we have a shaping reward $F_t$, and we want to utilize some optimality-preserving permutation of it $F'_t$ utilizing the potential of Equation \ref{initial_potential}, we can use
\begin{equation} \label{naive_conversion}
    F'_t = \begin{cases}  \sum_{n = 0}^{N-2} -\gamma^{n+1-N}F_n, &  \text{if } t = N - 1 \\  F_t, &  \text{if } t \neq N - 1,\end{cases}
\end{equation}
which is simply Equation \ref{shaping_reward} with $\Phi_t$ defined as in Equation \ref{initial_potential}. 

Equation~\ref{naive_conversion} has an appealing interpretation. It is equivalent to implementing the shaping reward ``normally'' until the very last time step, at which point the total discounted rewards are subtracted in order to ensure Equation \ref{boundary_condition} still holds. Described this way, it is both simple to understand and straightforward to implement.

Equation~\ref{naive_conversion} also has the advantage that it makes it particularly difficult for most agents to ``figure out'' that optimizing intrinsic motivation does nothing to increase their value function in the long run, because the adjustment term is at the very end of a given episode. This extends the reward horizon, to use the terminology of \citet{laud2004theory}, or the time delay between an action and the (intrinsic) returns dependent on that action. This makes it intentionally difficult for the agent to discover that IM doesn't ever actually affect the final return of an episode (because an appropriately discounted quantity will always be deducted later). Much work has gone into the goal of shortening the reward horizon on various problems, oftentimes through reward shaping terms (see, for example, Theorem 3 of \citet{Ngthesis}), but this work shows that actually \emph{increasing} the reward horizon for the futility of pursuing IM can be useful---it allows these rewards to still give hints to the agent, without being immediately discovered as ``worthless'' in the long run. The agent will then seek these rewards in the short term, but discard them in the long term insofar as optimizing for them would deviate from an optimal policy.

For the formal proof that Equation \ref{naive_conversion} leaves an optimal policy unaltered, we must make a single assumption about $F_t$ that limits the scope of rewards our method applies to:
\begin{assumption} \label{not_empowerment}
    $F_t$ is constant w.r.t.\ $a_{t'>t} \forall t, t' \in (0,1,...N-1).$
\end{assumption}
We refer to this assumption as ensuring that $F_t$ is ``future-agnostic:'' that is, it depends only on past actions of the agent in this episode, rather than future ones. This assumption is quite general, and holds for the majority of IM in the literature, including state-of-the-art exploration methods such as ICM and RND. Note that this assumption generally holds for action-dependent IM, so long as that action-dependence does not extend to \emph{future} actions, but is restricted to actions taken by the agent at the current time step and/or prior ones. The key example in the literature for which this assumption does not hold is empowerment \cite{mohamed2015variational}, in which states are given intrinsic weight that is based in part on future actions. Our method, however, can accomodate the vast majority of IM in the literature.

We can now prove Theorem \ref{theorem_2}:
\begin{thm}[PBIM Preserves Optimality]\label{theorem_2}
The addition of a shaping reward $F'_t$ of the form in Equation \ref{naive_conversion} leaves the set of optimal policies unchanged if Assumption \ref{not_empowerment} holds.
\end{thm}
\begin{proof}
    The potential of Equation \ref{naive_conversion} takes the form of Equation \ref{initial_potential}. With this choice of potential, the left side of Equation \ref{boundary_condition} becomes, in expectation,
\begin{align}
    U_t^\pi - \frac{U_0^\pi}{\gamma^t} &{= \sum_{n = t}^{N-1} \gamma^n F_n  - \sum_{n = 0}^{N-1} \gamma^{n-t} F_n }\\
    &{= \left( \sum_{n = t}^{N-1} \gamma^n F_n - \sum_{n = t}^{N-1} \gamma^n F_n\right) - \sum_{n=0}^{t-1} \gamma^{n-t} F_n }\\
    &{= - \sum_{n=0}^{t-1} \gamma^{n-t} F_n.}
\end{align}
This term depends simply on the discounted sum of all the IM rewards up to, but not including, $F_t$. From Assumption \ref{not_empowerment}, this has no $a_t$-dependence, and thus Equation \ref{boundary_condition} holds. From this, Theorem \ref{extending_thm} can be applied to show the optimal policy set is unchanged.
\end{proof} 

While we have just shown it does conserve the optimal policy, this form of PBIM has some potentially undesirable effects in practice. In particular, it may still bias the agent in the short term not only to prefer intrinsic rewards, but also to (temporarily) learn some false relationships between the reward distribution of the state space that then need to be unlearned. Particularly, if $F_t$ is consistently positive (as is the case with most exploration-based IM), then $F'_t$ will also tend to be consistently positive, \emph{except for in the last time step of an episode}, where it will be extremely negative, in order to offset the cumulative positive reward. This may cause an agent to initially learn that areas of the state space towards the end of an episode are ``bad'', and areas towards the beginning of an episode are ``good.'' While Theorem \ref{theorem_2} ensures these associations will eventually be unlearned, we would prefer to not learn them to begin with, as they may needlessly slow down training: particularly in exploration-focused environments, where they are often precisely the opposite of true. To mitigate this potential issue, we introduce a normalized variation of PBIM by replacing $F_t$ with
\begin{equation} \label{conversion}
F'_t = \begin{cases}  \sum_{n = 0}^{N-2} -\gamma^{n+1-N}F'_n, &  \text{if } t = N -1\\  F_t - \bar{F}, &  \text{if } t \neq N-1, \end{cases}
\end{equation}
where $\bar{F}$ is the expectation value of $F$ across prior training. This modified form ensures that the expected IM for both final and non-final time steps is 0, and thus these undesirable associations will not occur. In practice, $\bar{F}$ is calculated by taking a running mean of the previous intrinsic rewards collected across all workers during a single training epoch. 

The corresponding potential for Equation \ref{conversion} is
\begin{equation} \label{phi_set}
    \Phi_n = \begin{cases} \frac{-U'_0}{\gamma^N}, & \text{if } n = N \\ -U_n -\frac{\bar{F}}{\gamma - 1}, & \text{if } n \neq N, \end{cases}
\end{equation}
where $U'_0$ is the cumulative discounted mean-adjusted intrinsic return.
 The first case of this correspondence follows straightforwardly from the definition of the intrinsic return. The second case gives us back:
\begin{align}
    F'_{n \neq N} &= \gamma \Phi_{n+1} - \Phi_n \\
             &= \gamma ( - U_{n+1} - \frac{\bar{F}}{\gamma - 1}) + U_n + \frac{\bar{F}}{\gamma - 1}\\
             &= F_n - \frac{\gamma \bar{F}}{\gamma - 1} + \frac{\bar{F}}{\gamma - 1} \\
             &= F_n - \frac{\gamma-1}{\gamma-1}\bar{F} \\
             &= F_n - \bar{F}.
\end{align}
Because this $F_{n \neq N}$ term has an expected value of zero, $F_N$ will then similarly have an expected value of zero, as it will simply be the discounted sum of quantities with expectation zero. Additionally, because in either case, Equation \ref{conversion} differs from Equation \ref{naive_conversion} only by the addition of a constant factor of $\bar{F}$, and $\bar{F}$ is never dependent on the action $a_t$, this formulation satisfies Equation \ref{boundary_condition} as well. So we now have two formulations -- Equation \ref{naive_conversion} and Equation \ref{conversion} -- that can be used to implement IM without changing the optimal policy set of the underlying environment.

Figure \ref{fig:prior_table} shows PBIM in comparison to prior work in PBRS. Note that it allows for theoretical extensions to PBRS to be claimed by later works that are compatible with those extensions as long as those later works do not incorporate some element that explicitly eliminates those extensions. So for example, the extensions to POMDPs made by \cite{eck2016potential} are credited to \cite{grzes2017reward}, though the latter work doesn't explicitly reference POMDPs. Note also that our work here is unique among PBRS extensions in that it both conserves optimality and comes with a plug-and-play method to avoid hand-designing valid optimality-preserving shaping rewards.

\begin{figure}
    \centering
    \includegraphics[width=0.8\columnwidth]{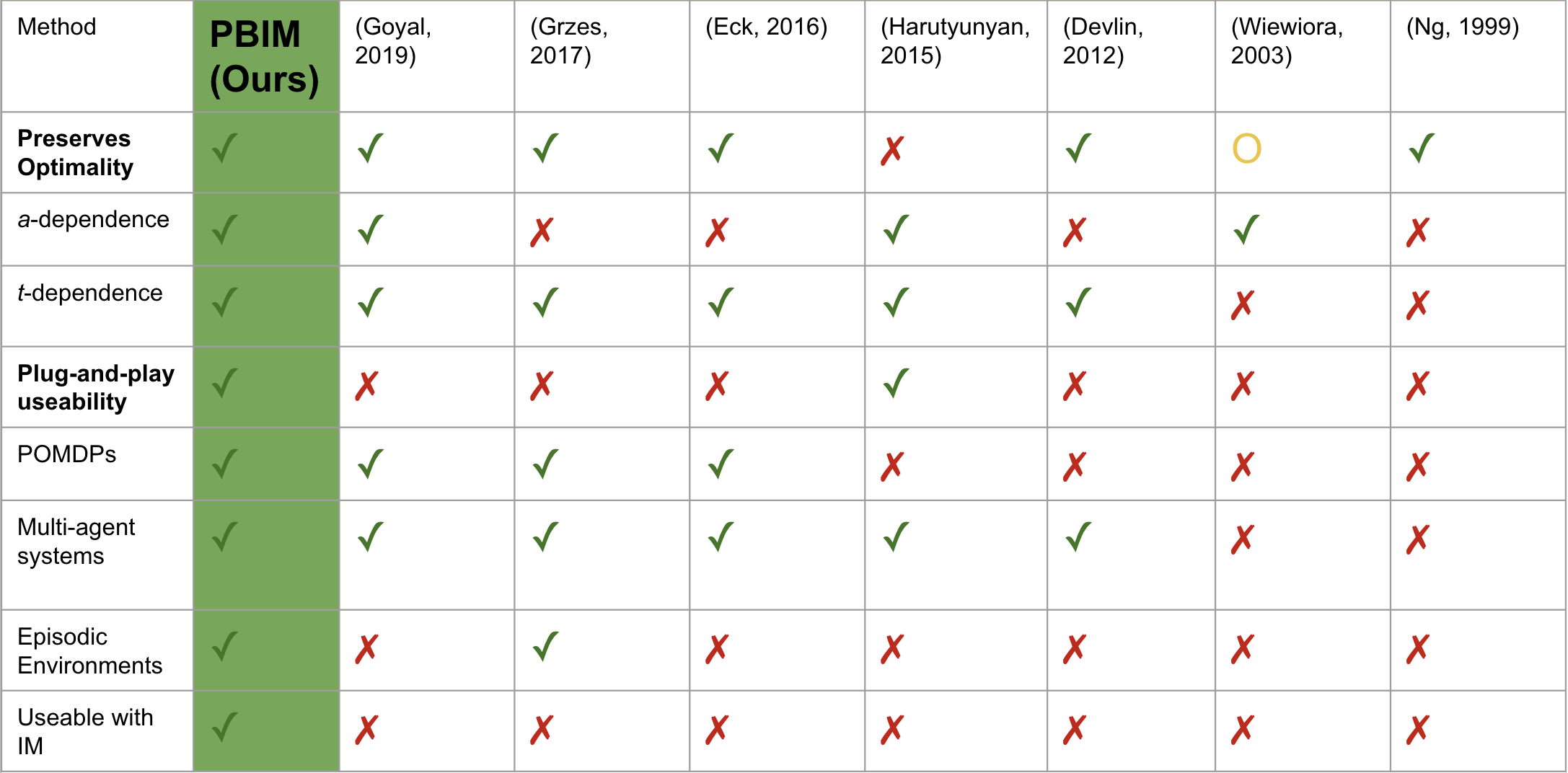}
    \caption{A comparison of our method with prior work in PBRS. A check here means that the domain or characteristic in question is treatable with or true of the method in question, while an 'X' indicates that it is not. The circle in the column for \cite{wiewiora2003potential} indicates that optimality is preserved, but only if the resulting potential is added back to the Q table during policy training.}
    \label{fig:prior_table}
\end{figure}

In the next section, we will generalize PBIM to a class of similar optimality-preserving plug-and-play reward shaping methods, using Generalized Reward Matching.
\section{Generalized Reward Matching: A General Class of PBIM-Like Conversion Methods}

Here we extend the theory of optimality-preserving shaping rewards to a more general form than that covered by either PBIM, or by previous literature. We introduce Generalized Reward Matching (GRM), which we prove is as general as any PBRS method yet published. To state this more precisely: for every optimality-preserving potential one could construct by hand, there is an equivalent GRM reward function. 

\subsection{Introducing the Matching Function}

GRM leverages the notion of a ``matching function,'' 
\begin{align}\label{matching_function}
    m_{t, t'}: N \times N \to [0,1],
\end{align}
which `matches' the shaping rewards received at time~($t'$) to their later-subtracted counterparts at time~($t$). The matching function $m_{t, t'}$ represents, at time step $t$, the proportion of the shaping reward originally received by the agent at some earlier time step $t'$ that is due to be ``matched'' (subtracted with an appropriate discount).

We define the general class of GRM transformation functions as
\begin{equation}\label{drawer}
    F^{\text{'GRM}}_t = F_t - \sum_{i = 0}^{t}\gamma^{i - t}F_{i}m_{t, i}.
\end{equation}

This is to say that, given a valid matching function, Equation \ref{drawer} can be used in a ``plug-and-play'' manner to convert any future-agnostic shaping function or IM $F_t$ into a GRM function $F^{\text{'GRM}}_t$.\footnote{As a reminder, ``future-agnostic'' shaping functions are those for which Assumption \ref{not_empowerment} is true.} We will prove that $F^{\text{'GRM}}_t$ is an optimality-preserving shaping reward given any future-agnostic $F_t$ if $m_{t, t'}$ meets two conditions.

First, $m_{t, t'}$ must be ``fully-matching", which is to say that each shaping reward $F_{t'}$ must be subtracted exactly once total:
\begin{equation}\label{Condition1}
  \forall_{t'} \sum_{j = t'}^{N-1} m_{j, t'} = 1.
\end{equation}
Note that, though all the implementations of $m_{t, t'}$ we explore empirically in this paper match the intrinsic reward for each time step at some other \textit{individual} time step, this condition allows for the general case in which the matching is spread out between multiple time steps. 

We also require that the matching function $m_{t, t'}$ be ``future-agnostic", similarly to our requirement for $F_t$ in Assumption \ref{not_empowerment}. For the matching functions, this means that it should not subtract intrinsic rewards in any time step before the time at which those rewards will actually be initially obtained by the agent: 
\begin{equation}\label{Condition2}
    \forall_{t,t'>t} \quad  m_{t, t'} = 0.
\end{equation}
Note then that the exponent in Equation~\eqref{drawer} will always be either zero or negative.

We will call a reward shaping function of the form in Equation~\eqref{drawer} that meets these conditions a {\em GRM shaping reward}. We will show that any PB-GRM shaping reward will not change the optimal policy set of the underlying environment. In Section \ref{proof_section}, we first motivate an intuition for GRM, situate it relative to PBIM, and then prove this optimality-preserving property.

\subsection{GRM Preserves Optimality} \label{proof_section}
To get an intuitive sense of GRM, and the matching function in particular, it is useful to examine what happens when we choose $m_{t, t'}$ such that
\begin{equation}\label{trivial_m}
     m_{t, t'} = \begin{cases}
        1 & \text{if} \quad t = t' \\
        0 & \text{otherwise}.
    \end{cases}
\end{equation}
One can check that this meets the conditions in Equations~\eqref{Condition1}~\&~\eqref{Condition2}. We can interpret this matching function, loosely, as ``accounting for'' each reward fully in the timestep in which it is recieved by the agent. With this matching function, Equation~\eqref{drawer} becomes $F'_t = 0 \quad \forall_t$, something that trivially conserves the optimal policy of any environment, even if it may be of dubious use.\footnote{It effectively eliminates any intrinsic reward, preventing any reduction in training time from simply using the base reward alone.}

Similarly, we can connect GRM back to PBIM through Equations~\eqref{naive_conversion} and ~\eqref{drawer}. If we use the matching function 
\begin{equation}\label{pbimmatching}
    m_{t, t'} = \begin{cases}
        1 & \text{if} \quad t = N-1 \\
        0 & \text{otherwise},
    \end{cases}
\end{equation}
we are accounting for each intrinsic reward at the very end of the episode, and no sooner. Using this matching function, we can see that Equation~\eqref{drawer} becomes equal to Equation~\eqref{naive_conversion}. PBIM, then, is equivalent to the specific case of GRM with a matching function defined by Equation~\eqref{pbimmatching}. This also preserves optimality, as we proved in Section \ref{converting}.

Of course, Equations~\eqref{trivial_m} and \eqref{pbimmatching} don't exhaust the possible choices for $m_{t, t'}$, but they should provide the reader with some initial intuition both of what $m_{t, t'}$ represents, and of its potential to represent a general class of shaping functions. We will prove that not only does this class of reward shaping functions preserve optimality, but it in fact is equivalent to the class of possible potential-based reward shaping functions. 

We will proceed by showing that, for every $F^{\text{'GRM}}_t$ that is a valid implementation of Equation~\eqref{drawer},\footnote{Here, and for the rest of the paper, we take a ``valid" implementation of Equation~\eqref{drawer} to be one in which Equations~\eqref{Condition1}~\&~\eqref{Condition2} hold for $m_{t, t'}$ and Equation~\eqref{not_empowerment} holds for $F_t$.} there exists some associated $\Phi_t$ such that $F^{\text{'GRM}}_t$ equals Equation~\eqref{shaping_reward}, and this chosen $\Phi_t$ meets the condition in Equation~\eqref{boundary_condition}: a condition which we proved in Theorem \ref{extending_thm} to be sufficient for preserving optimality of the underlying environment.

\begin{theorem} \label{PBGRM_theorem}
For every valid GRM shaping reward, there is a corresponding equivalent potential-based shaping reward that conserves the optimal policy set of the underlying MDP.
\end{theorem}
\begin{proof}

Given a GRM shaping function, we can construct a potential function 
\begin{equation} \label{pot_4_m_f}
\Phi_t = \begin{cases}  C &  \text{if } t = N \\ \sum\limits_{j=t}^{N-1} \sum\limits_{i = 0}^{j} ( \gamma^{i-t}F_{i}m_{j, i}) - U_t^I + \gamma^{N-t}  C &  \text{if } t \neq N, \end{cases}
\end{equation}
where $U_t^I = \sum\limits_{i = t}^{N-1}\gamma^{i-t}F_{i}$ is the cumulative discounted (intrinsic) return of $F_t$ from time $t$ until the end of the episode, and $C$ is an arbitrary constant.\footnote{While a proof by construction omitting $C$ would be sufficient to demonstrate that for every GRM shaping function, there exists at least \textit{one} corresponding optimality-preserving potential function (and thus that GRM preserves optimality), this more general form draws attention to the fact that the potentials in PBRS are always, effectively, only defined up to a constant. As we will see shortly, this constant drops out of all the relevant equations, and thus can be chosen freely without loss of generality.} It's important to note that, similarly to in Section \ref{converting}, we do not actually need to \textit{know} what $U_t^I$ is in order to use it as a component of our potential, as all of the difficult-to-obtain components of it (i.e. yet-to-be-obtained rewards) will drop out in the derivation.

This allows us to construct a potential-based reward $F_t^{\Phi'}$ which is equal to $F^{\text{'GRM}}_t$. Using Equations~\eqref{shaping_reward}~\&~\eqref{pot_4_m_f}, 
we can construct the potential-based reward:
\begin{align}
F_t^{\Phi'} = \gamma  \Phi_{t+1} - \Phi_t 
 & = \gamma \Biggl( \sum_{j=t+1}^{N-1} \sum_{i = 0}^{j} \gamma^{i - t - 1}  F_{i} m_{j,i} \Biggr. \nonumber
 \Biggl. -U_{t+1}^I + \gamma^{N-t - 1}  C \Biggr) \nonumber \\
 & - \Biggl( \sum_{j=t}^{N-1} \sum_{i = 0}^{j} \gamma^{i- t}  F_{i}  m_{j,i}  - U_t^I  + \gamma^{N-t}  C \Biggr)\\
 & = U_t^I - \gamma U_{t+1}^I  - \sum_{i = 0}^{t}\gamma^{i - t}  F_{i}  m_{t, i} 
 = F_t - \sum_{i = 0}^{t}\gamma^{i- t}  F_{i}  m_{t, i} 
  = F^{\text{'GRM}}_t.
\end{align}

So the PBRS reward using this potential is equal to the GRM reward. Given Theorem \ref{extending_thm}, all that's now needed to prove that $F^{\text{'GRM}}_t$ doesn't alter the optimal policy is to check the boundary condition in Equation~\eqref{boundary_condition}. Dropping the expectation without loss of generality, this becomes
\begin{equation}
     \gamma^{N-t}\Phi_N -  \Phi_t = \Phi'_t 
    = U_t^I - \sum_{j=t}^{N-1} \sum_{i = 0}^{j} \gamma^{i-t}F_{i}m_{j,i} 
    = U_t^I - \sum_{j=t}^{N-1} \sum_{i = 0}^{N-1} \gamma^{i-t}F_{i}m_{j,i}\label{deriv2}.
\end{equation}
Note that in the final equality of Equation~\eqref{deriv2}, we used Equation~\eqref{Condition2} to increase the upper bound of the inner sum, as all terms here with  $i > j$ are zero by definition.
We can then swap the order of the sums, and now have
\begin{align}
    U_t^I - \sum_{i = 0}^{N-1} \sum_{j=t}^{N-1} \gamma^{i-t}F_{i}m_{j,i} = 
    U_t^I - \sum_{i = 0}^{N-1} \gamma^{i-t}F_{i} \sum_{j=t}^{N-1} m_{j,i} & = \nonumber \\
    U_t^I - \sum_{i = 0}^{N-1} \gamma^{i-t}F_{i} \Biggl( \sum_{j=i}^{N-1} m_{j,i}  - \sum_{j=i}^{t-1} m_{j,i}\Biggr) & = \nonumber \\
     U_t^I - \sum_{i = 0}^{N-1} \gamma^{i-t}F_{i}\Biggl(1-\sum_{j=i}^{t-1} m_{j,i}\Biggr) & \label{deriv3}. 
\end{align}
To get Equation~\eqref{deriv3}, we applied 
Equation~\eqref{Condition1} to reduce the equivalent inner sum to 1. Expanding out the intrinsic return, this quantity then becomes
\begin{align}
    \sum_{i = t}^{N-1}\gamma^{i-t}F_{i} - \sum_{i = 0}^{N-1} \gamma^{i-t}F_{i}&\Biggl(1-\sum_{j=i}^{t-1} m_{j,i}\Biggr) = \nonumber 
- \sum_{i = 0}^{t-1} \gamma^{i-t}F_{i} +\sum_{i = 0}^{N-1} \gamma^{i-t}F_{i}\Biggl(\sum_{j=i}^{t-1} m_{j,i}\Biggr) \nonumber \\
= & - \sum_{i = 0}^{t-1} \gamma^{i-t}F_{i} +\sum_{i = 0}^{t-1} \gamma^{i-t}F_{i}\Biggl(\sum_{j=i}^{t-1} m_{j,i}\Biggr).\label{deriv4}
\end{align}

Note that to get Equation~\eqref{deriv4}, we once again applied Equation~\eqref{Condition2}, this time to reduce all terms in a sum with $j>t-1$ to zero. 

We have now reduced this expression to a quantity which depends only on the values of $F_{i}$ where $i<t$. Given our condition that 
$F_{t}$ satisfies Equation~\eqref{not_empowerment}, this quantity is necessarily independent of $a_t$, and thus is a valid $\Phi'_t$ for Equation~\eqref{boundary_condition}. Given Theorem \ref{extending_thm}, this entails that any valid $F^{\text{'GRM}}_t$ is also a PBRS function that conserves the optimal policy set of the underlying environment.
\end{proof}
\begin{corollary} Every valid GRM shaping reward leaves the set of optimal policies of the underlying environment unchanged.
\end{corollary}

In Section \ref{sec:equivalence_proof}, we will additionally prove that not only is GRM guaranteed to conserve the optimal policy, but \textit{every} optimality-preserving PBRS term can be written in terms of a GRM reward. That is to say, any hand-designed shaping reward writable in the form of Equation \ref{shaping_reward} that conserves the optimal policy is mathematically equivalent to some GRM matching function and initial reward function.

\subsection{PB-GRM and PBRS Are Equivalent}\label{sec:equivalence_proof}

Here we prove the equivalence of GRM and Equation \ref{shaping_reward}, which we take to be the most general formulation of PBRS. We will rely on Theorem~\ref{PBGRM_theorem} to show that for each GRM shaping reward, there exists a corresponding PBRS shaping reward. We additionally prove that for each PBRS shaping reward, there is a corresponding GRM shaping reward: In other words, there exist no optimality-preserving PBRS functions that GRM cannot accommodate equivalently to their potential-based formulation. Thus, we prove that the two are equivalent.
\begin{theorem}\label{proof1}For every potential-based shaping reward that conserves the optimal policy set of the underlying environment, there is an equivalent corresponding valid GRM shaping reward.
\end{theorem}
\begin{proof}
Let's say that we have a potential-based shaping reward $F_t^\Phi$ in the form of Equation~\eqref{shaping_reward} following the condition in Equation~\eqref{boundary_condition}: in other words, one that preserves optimality (from Theorem \ref{extending_thm}).  We may then freely choose any matching function $m_{t,t'}$ under the condition that
\begin{equation} \label{proof_m_condition}
    m_{t, t} \neq 1 \quad  \forall t.
\end{equation}
As an important clarification, this condition is not necessary for matching functions in general. However, it greatly simplifies our proof, which requires only that we show \textit{some} combination of $F_t$ and $m_{t,t'}$ that form a $F^{'GRM}_t$ equivalent to $F_t^\Phi$, not exhaust all possibilities. Our chosen $m_{t,t'}$, of course, also has to follow the conditions in Equations~\eqref{Condition1} and \eqref{Condition2}, but these do not practically constrain our ability to choose a valid matching function.

Once $m_{t,t'}$ is chosen, then given some $\Phi$, we can construct
\begin{equation}\label{f_from_m_phi}
    F_t = \frac{\gamma \Phi_{t+1} - \Phi_t + \sum_{i=0}^{t-1} \gamma^{i} F_{i-t} m_{t,i}}{1-m_{t, t}}
\end{equation}
(the denominator in this expression is the reason for our stipulation in Equation~\eqref{proof_m_condition}). We can then verify that with $F_t$ defined as such, the $F^{\text{'GRM}}_t$ of Equation~\eqref{drawer} is equal to $F_t^\Phi$:
\begin{align}
    F^{\text{'GRM}}_t &= F_t - \sum_{i = 0}^{t}\gamma^{i-t}F_{t'}m_{t,i} \\
    &= F_t \left( 1 - m_{t,t} \right) - \sum_{i = 0}^{t-1}\gamma^{i-t}F_{t'}m_{t,i} \\
    &= \gamma \Phi_{t+1} - \Phi_t + \sum_{i=0}^{t-1} \gamma^{i-t} F_{i} m_{t,i} - \sum_{i = 0}^{t-1}\gamma^{i-t}F_{i}m_{t,i} \\
    &= \gamma \Phi_{t+1} - \Phi_t = F_t^\Phi.
\end{align}
Thus, any potential-based optimality-conserving function is also expressible within the GRM framework. 
\end{proof}
Together with the proof of the inverse in Theorem~\ref{PBGRM_theorem}, this suffices to show 
\begin{corollary}
    PBRS and GRM describe an equivalent set of shaping functions.
\end{corollary}

Though we have shown theoretically that all valid GRM shaping functions conserve the optimal policy, we have not shown as of yet any general rules for how to select a ``good" matching function given some environment and shaping reward. While this is a difficult question that has plagued PBRS literature in general, in Section \ref{sec:empirical}, we define a representative subset of possible matching functions\footnote{This subset includes Equation \ref{pbimmatching}, and thus includes PBIM.} and empirically validate some general heuristics for what makes an effective matching function. We demonstrate that both GRM and PBIM speed up training over the base IM being used, as well as helping the agent avoid converging to a policy that is suboptimal in the original environment.

\subsection{A Representative Subset of Plausible GRM Functions}

As there are infinitely many GRM matching functions to choose some, in order to empirically evaluate GRM as a class of methods, we have to select some representative plausible subset of these methods. We would like it to be a subset of possible matching functions that includes PBIM, with a matching function as defined as in Equation~\eqref{pbimmatching}. We define a subset of matching functions parameterized by a delay parameter $D$, where for the $D=0$ case the resulting matching function becomes Equation~\eqref{trivial_m}, while for $N-1 \leq D < \infty$ it becomes Equation~\eqref{pbimmatching}. We define
\begin{equation}\label{hyperparam_m}
    m_{t,t'}(D) = \begin{cases}
        1 & \text{if } t - t' = D \\
        1 & \text{if } t = N-1 \text{ and } t' > N - D\\
        0 & \text{otherwise}
    \end{cases}
\end{equation}
to be the class of valid matching functions parameterized by $D$. Intuitively, this matching function accounts for each reward either exactly $D$ time steps after it has been received by the agent, or in the final time step of the episode, whichever comes first.Implementing it with Equation~\eqref{drawer} gives the shaping function
\begin{equation}\label{hyperparam_tuning}
    F^{\text{'GRM}}_t(D) = \begin{cases}
        F_t & \text{if } t < D\\ 
        F_t - \gamma^{-D}F_{t-D} & \text{if } D \leq t < N-1\\
        \sum\limits_{i = 0}^{D-1} - \gamma^{i - D} F_{N - 1 - D + i}& \text{if } t = N-1.
    \end{cases}
\end{equation}
Intuitively, Equation~\eqref{hyperparam_tuning} subtracts the intrinsic reward obtained at time step $i$ later along the trajectory at time step $i+D$ or $N-1$, whichever comes first. If $D$ is very high, this becomes equivalent to PBIM, acounting for each intrinsic reward in the final time step of the episode. If on the other hand, we choose $D=1$ (the lowest non-trivial value of $D$), the shaping reward becomes
\begin{equation}
   F'_t = \begin{cases}
        F_t & \text{if} \quad t  = 0 \\ 
        F_t - \frac{F_{t-1}}{\gamma} & \text{if} \quad 0 < t < N-1\\
        -\frac{F_{t-1}}{\gamma} & \text{if} \quad t = N-1,
    \end{cases}
\end{equation}
which essentially implements Equation~\eqref{shaping_reward} with $\Phi_t=\frac{F_t}{\gamma}$  at all time steps but the first and last.

In Section \ref{sec:empirical}, we will empirically demonstrate the efficacy of methods sampled from this parameterization of GRM at both speeding up training and preventing the agent from converting to a suboptimal policy.

\section{Empirical Results}\label{sec:empirical}

Here, we empirically demonstrate the efficacy of PBIM/GRM for both
an exploration-based tabular IM reward and for RND \cite{burda2018exploration}. For these results, we focused on environments with a knowable set of optimal policies where the base IM demonstrably alters performance and other IM approaches not guaranteeing optimality fail to converge towards an optimal solution. The former of these demonstrations most clearly shows GRM’s potential to speed up convergence when compared to either a baseline IM or no IM, while the latter demonstrates GRM’s ability to preserve an agent’s convergence to an optimal policy, even with an IM that would otherwise explicitly alter the optimality of that policy.

\subsection{Minigrid Doorkey}\label{minigrid}

We demonstrate an improvement in both speed of convergence and performance of the converged-to policy when using our method to confer optimality guarantees to a tabular exploration reward term in the MiniGrid DoorKey 8x8 environment \cite{minigrid}. This environment, an example of which is pictured in Figure~\ref{doorkey}, challenges the agent to reach a goal state in the bottom-right corner by picking up a key, carrying it to a door, then unlocking that door. The environment itself has sparse rewards, returning a reward of 1 for successfully reaching the goal and 0 for every other transition. It is also partially observable, as the agent can only see at most in a 7x7 grid ``in front'' of it. The maximum episode length in this environment is 640 steps.

\begin{figure}
    \centering
    \includegraphics[width=0.8\columnwidth]{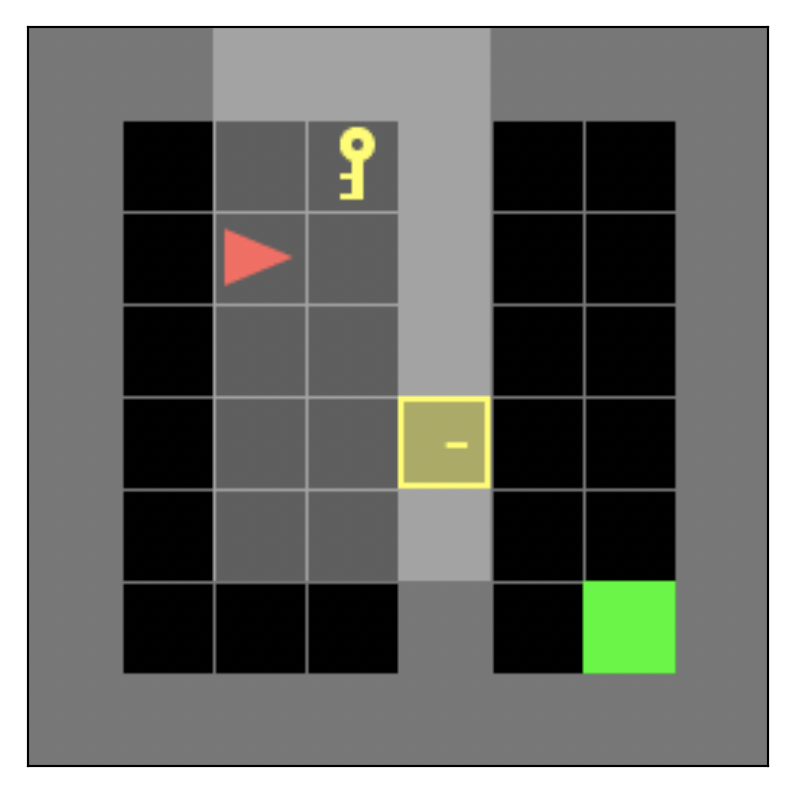}
    \caption{An example MiniGrid DoorKey 8x8 environment.}
    \label{doorkey}
\end{figure} 

We used a tabular exploration reward of the form 
\begin{align}\label{tabular_im}
   F_t = \frac{\alpha}{n(s)}, 
\end{align}
where $n(s)$ is the number of times a state has been previously visited within an episode, and $\alpha$ is a coefficient controlling the magnitude of exploration reward relative to the environment reward. A tabular reward was chosen for simplicity, given the simple environment, and also for explainability, as it is much clearer to analyze its alterations to the optimal strategy than it would be for intrinsic motivation methods based, for example, on the output of a neural network. 
This particular tabular reward form appears in previous literature \cite{strehl2008analysis, burda2018exploration}.   

As the environment itself is partially observable, and we wish to demonstrate the versatility of our method when applied to reward functions with dependence on arbitrary variables, we defined the ``state'' counter $n_t$ not based on the agent's observation space, but on information internal to the environment itself regarding the agent's position and whether it was holding the key. So, for example, the first time the agent visited state $\{3,4,0\}$ (\{``3rd vertical position,'' ``4th horizontal position,'' and ``not carrying the key''\}), it would receive a reward of $1$, followed by $\frac{1}{2}$, $\frac{1}{3}$, etc.\footnote{To avoid a combinatorial explosion, the agent's direction and the door's status (locked/unlocked/open) were not incorporated into this reward.}

It is worth expanding on the some of particular ways that the IM in this experiment changed the set of optimal policies in this environment. With Equation \ref{tabular_im}, the particular intrinsic reward we used in this environment---though certainly not unique to it; see \cite{burda2018large} as well as Section \ref{small_cliffwalker}---there is a danger of the agent ``procrastinating,'' and collecting more exploration rewards than are strictly necessary to reach the goal state, rather than taking the most efficient action at each time step. Let's examine a simple version of this: if an agent can either take the most efficient action to proceed toward the goal state, or wait for $t$ time steps before doing so (say, by either alternating ``turn left'' and ``turn right'' actions, or by taking the ``toggle'' action repeatedly when there is no door to toggle), then which course of action is preferred to the other depends on its intrinsic reward for the current grid tile. If it has visited its current tile a total of $n$ times (including the current time step), then stalling for $t$ time steps before proceeding to the goal becomes strictly preferred to the optimal action if the condition
\begin{equation} \label{procrastination_condition}
    1-\gamma^t < \sum_{t' = 0}^{t-1}\alpha \gamma^{t'}\frac{1}{n+t'+ 1}
\end{equation}
is met. Here, the left side of Equation \ref{procrastination_condition} refers to the cost of delaying the step at which it will reach the goal state by $t$ time steps, and the right side refers to the intrinsic reward gained by doing the same. 

Thus the set of optimal policies in this environment can be altered by the shaping reward in Equation 41 for sufficiently high values of $\alpha$ and $\gamma$. To test how violations of this inequality affect performance, we tested at values of $\alpha$ and $\gamma$ which vary in the frequency with which we can expect Equation~\ref{procrastination_condition} to be satisfied. In particular, higher values of both $\alpha$ (a positive coefficient in the right side of Equation \ref{procrastination_condition}) and $\gamma$ (positive on the right side of Equation \ref{procrastination_condition} and negative on the right side) will increase the frequency of states in which the optimal policy will be altered by this intrinsic reward.

We iterated on the initial codebase of \cite{minigrid_code} for our implementation.

We used the Proximal Policy Optimization (PPO) algorithm as introduced by \citep{schulman2017proximal}. We included a Long Short-Term Memory (LSTM) layer \cite{hochreiter1997long}, a type of recurrent layer, in the network architecture to deal with the partially-observable, non-Markovian nature of the environment.

We tested eight reward schemes: a control that recieved no intrinsic rewards, base intrinsic rewards without any PBRS, non-normalized PBIM as implemented in Equation~\ref{naive_conversion}, normalized PBIM as implemented in Equation~\eqref{conversion}, GRM as implemented in Equation~\eqref{hyperparam_tuning} with $D=10$, GRM with $D=1$, and a normalized version of both of these latter methods. For the normalized methods, we used a cumulative average of the rewards across all episodes and workers initialized at the beginning of training to compute $\bar{F}$. We consider a policy to converge when it does not truncate an episode without reaching the goal state during the last 1,000 recorded steps.

We ran 16 processes concurrently for 10 million cumulative steps, and recorded performance every 128 steps. 

 We use episode length as a proxy for performance, as the (extrinsically) optimal policy of the environment is solely to reach the goal state as quickly as possible. Episode lengths for each of our sets of tested hyperparameters are depicted in Figures~\ref{frame_results_005},~\ref{frame_results_02},~\&~ \ref{fig:025-full}, respectively. Shaded regions here represent standard deviation among the 16 processes, rather than error on the mean: this is expected to be high, because MiniGrid environments are procedurally generated and variance in the optimal path length from one episode to another is expected. Table~\ref{table_minigrid} contains the time to convergence $T$ (if converged, N/A if they did not converge), mean episode length $\bar{N}$, and standard deviation $\sigma$ after convergence for each reward scheme. $\bar{N}$ and $\sigma$ were calculated from the last 1,000 data points. $T$ was determined by the first time step in which the average episode length falls below $\bar{N}$ for that run.
 
 For each pair of results in each domain we performed a 1-sided T-test, and
notate all statistically significant differences from the best-performing converged policy in Table \ref{table_minigrid}.

\begin{figure*}[t]
    \centering
    \begin{subfigure}{0.49\textwidth}
        \centering
        \includegraphics[width=\textwidth]{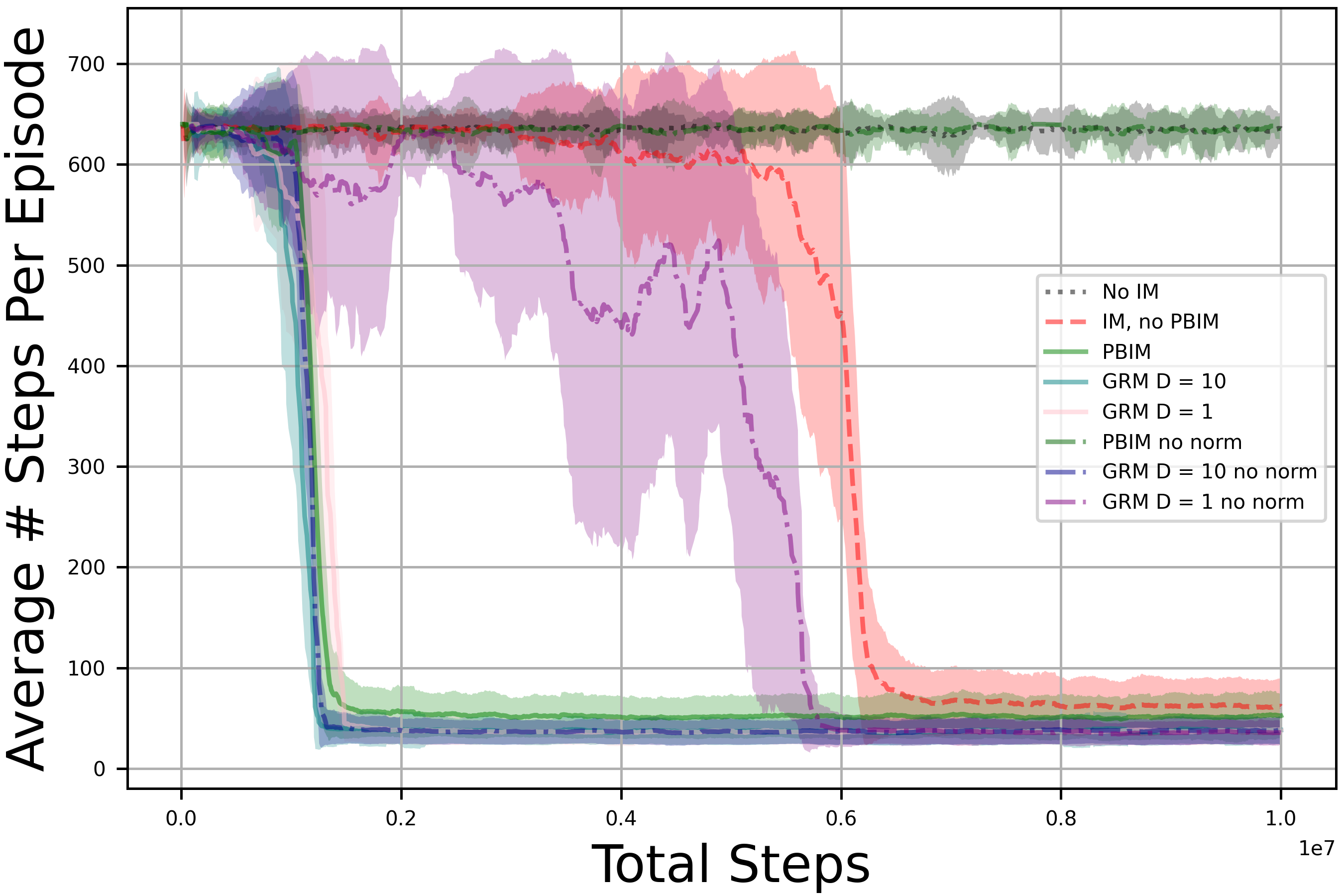}
        \caption{$\alpha = .025$, $\gamma = .995$}
        \label{fig:025-full}
    \end{subfigure}
    \begin{subfigure}{0.49\textwidth}
        \centering
        \includegraphics[width=\textwidth]{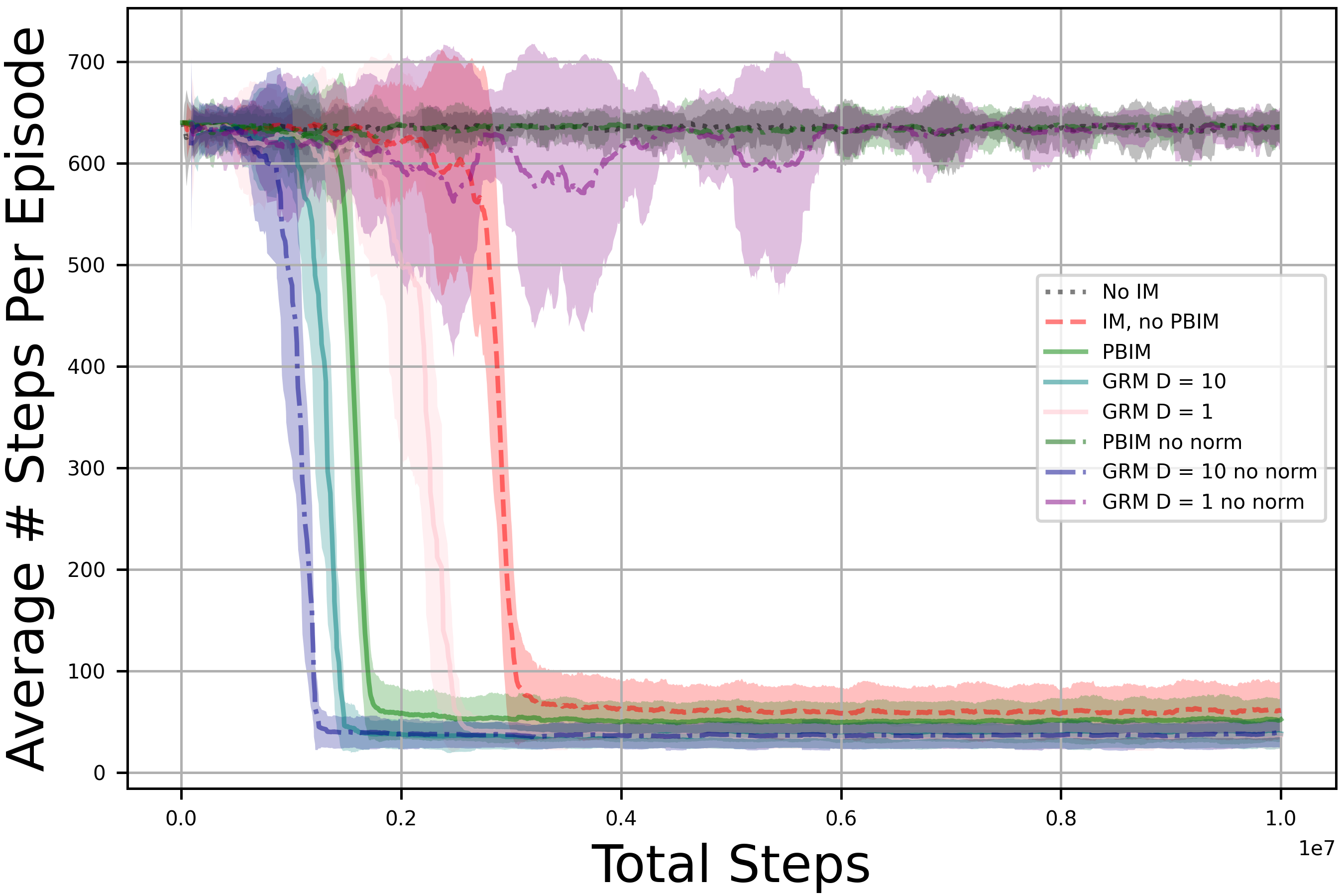}
        \caption{$\alpha = .02$, $\gamma = .995$}
        \label{frame_results_02}
    \end{subfigure}
    \begin{subfigure}{0.49\textwidth}
        \centering
        \includegraphics[width=\textwidth]{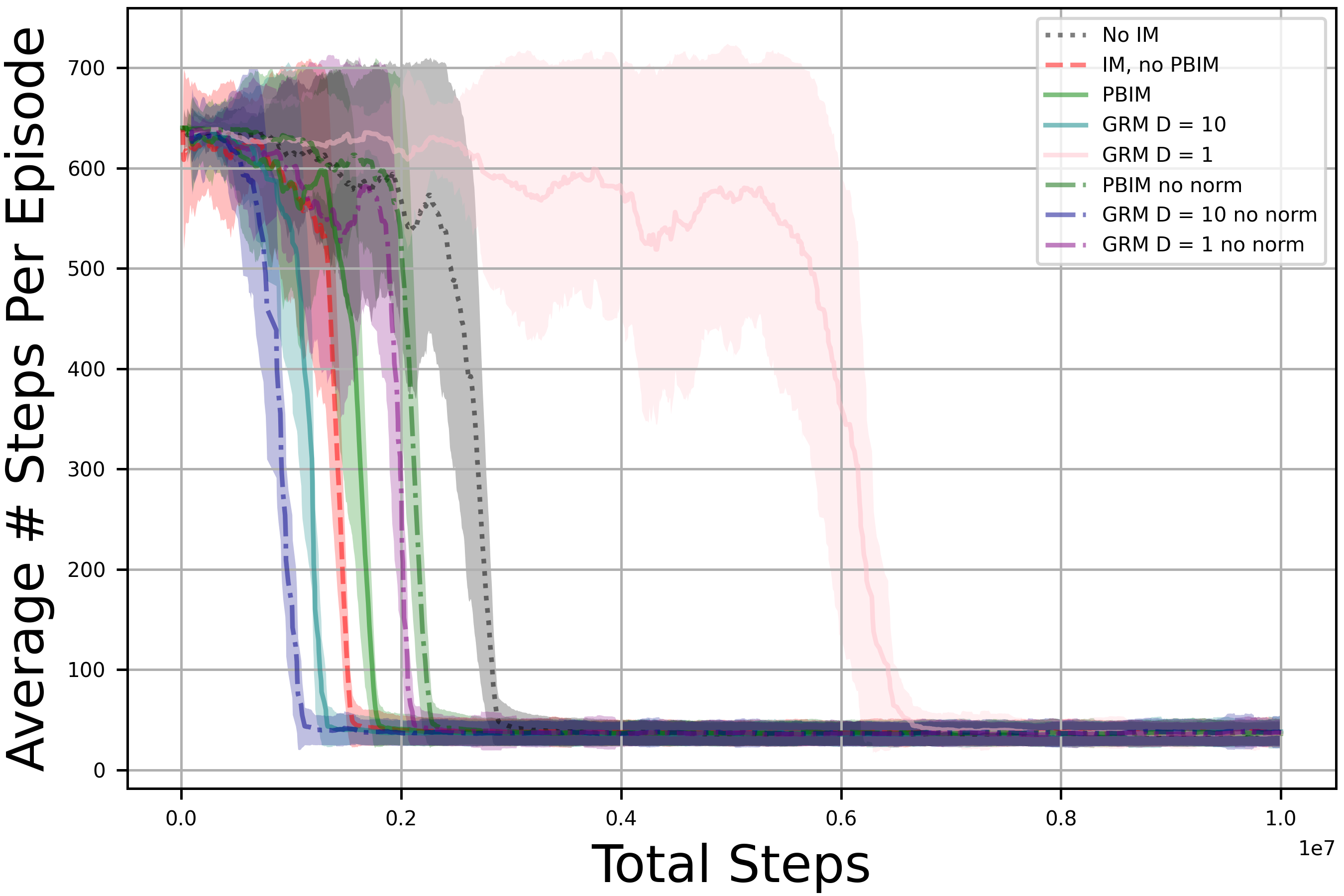}
        \caption{$\alpha = .005$, $\gamma = .99$}
        \label{frame_results_005}
    \end{subfigure}
    \caption{(\subref{fig:025-full}), (\subref{frame_results_02}), \& (\subref{frame_results_005}) Frames per episode for each method (lower is better). The shaded region represents standard deviation, and plots are of a 100-point moving average.}
    \label{frame_results}
\end{figure*}


%
\begin{table*}[t]
    \begin{center}
        \begin{small}
            \begin{sc}
                \begin{tabular}{lccc|ccc|ccc}
                    \toprule
                    & \multicolumn{3}{c|}{$\alpha=0.005$, $\gamma = 0.99$} & \multicolumn{3}{|c|}{$\alpha = 0.02$, $\gamma = 0.995$} & \multicolumn{3}{|c}{$\alpha = 0.025$, $\gamma = 0.995$} \\
                    \midrule
                    & $T$ & $\bar{N}$ & $\bar{\sigma}$ & $T$ & $\bar{N}$ & $\bar{\sigma}$ & $T$ & $\bar{N}$ & $\bar{\sigma}$ \\
                    \midrule
                     NO IM   & 2.95e6&\textbf{35.9}&  12.0    &           N/A&634.8&  18.55    &         N/A&634.8&  18.6      \\
                    IM & 1.46e6&\textbf{37.3}&13.1&                       2.88e6&60.5&27.5 &              6.07e6& 62.0&27.9\\
                    PBIM    & 1.67e6& \textbf{36.4}&12.5 &                  1.67e6&51.8 & 22.4 &            1.26e6&51.2& 22.7\\
                    PBIM NO NORM   & 2.29e6& \textbf{37.1 }& 13.4&            N/A & 635.9&14.9 &              N/A & 634.7 &18.8\\
                    GRM D=10    & 1.85e6& \textbf{36.8}&13.4 &              1.50e6&\textbf{37.1}& 13.1 &            1.19e6&\textbf{38.0}& 14.6\\
                    GRM D=10 NO NORM & \textbf{1.01e6}&\textbf{37.4}&13.6&         \textbf{1.17e6}&\textbf{37.6}&13.7 &              \textbf{1.18e6}& \textbf{37.6} &13.3\\
                    GRM D=1    & 8.71e6& \textbf{36.9} & 14.1         &   2.62e6 & \textbf{36.3}&12.9  &                     1.67e6 & \textbf{37.5}& 13.4\\
                    GRM D=1 NO NORM  & 5.39e6& \textbf{36.2} &12.8   &   N/A & 634.5&   19.4&                    5.86e6 & \textbf{35.7} &  12.1      \\

                    \bottomrule
                \end{tabular}
            \end{sc}
        \end{small}
    \end{center}
    \caption{Time to convergence ($T$), mean steps per episode after convergence ($\bar{N}$), and average standard deviation of steps per episode after convergence ($\bar{\sigma}$) for three parameter settings. Lower $\bar{N}$ and $T$ are better. Fastest $T$ in bold, and lowest $\bar{N}$ or any method with an $\bar{N}$ statistically indistinguishable from that run at $p = .05$ are in bold (differences between bold and non-bold $\bar{N}$ are statistically significant).}
    \label{.005} \label{.02} \label{.025}\label{table_minigrid}
\end{table*}

\subsubsection{Discussion: PBIM and GRM Outperform Baseline IM}

As can be seen in Figure~\ref{frame_results} and Table~\ref{table_minigrid}, PBIM and GRM both consistently outperform the baseline IM method in terms of converging to an extrinsically better policy. Additionally, when the IM most often changes the optimal policy (Figure~\ref{fig:025-full}), all but one of our methods converge faster than the baseline, and as predicted, the degree to which they outperform the baseline is correlated with the frequency with which the optimal policy is altered (compare to Figure~\ref{frame_results_02}).

The only experiment in which our (normalized) methods did not all consistently outperform the baselines in both speed of convergence and final policy was with $\alpha = 0.005, \gamma = 0.99$. Here though, as can be seen in Table~\ref{table_minigrid}, not only did our best-performing method in the other experiments still outperform the baseline IM,\footnote{Our second-best-performing method in the other experiments technically converges more slowly here than the baseline IM according to Table \ref{table_minigrid}, but this is mostly an artefact of the way we have chosen to calculate convergence: this method learns how to consistently reach the goal at all before the baseline IM, but takes slightly longer than the baseline IM to lower its average time to reach that goal. Looking at Figure~\ref{frame_results_005}, it can be argued that this choice of metric does not do justice to the normalized form of GRM $D=10$, but we are reporting this definition of convergence regardless, as we chose it before running our experiments.} but all but one converged more quickly than the no-IM run. In the following sections, we will analyse and discuss the results for each of our three classes of methods separately.

\subsubsection{Discussion: PBIM}
 Though not providing as consistent or drastic of an improvement as our GRM-based methods here, PBIM successfully improves both convergence time and efficiency of the converged policy when compared with both baselines. Note also that our modification of Equation~\ref{naive_conversion} into Equation~\ref{conversion} was key in allowing for convergence in the more difficult environments of Figures~\ref{frame_results_02}~\&~\ref{fig:025-full}. The environment in which it did not improve convergence time over the IM baseline can be explained by noting, as discussed in Section~\ref{converting}, that there is a reward horizon for the lack of intrinsic rewards' utility in the PBRS agent, and in simpler environments (such as that with a lower $\gamma$ value), there is a risk of this reward horizon being successfully learned before the environment itself is fully solved. If this happens, PBRS-based methods can slow the speed of convergence over non-optimality-preserving IM, rather than increase it, by teaching the agent to ignore the IM term prematurely. We note, however, that both normalized and non-normalized PBIM converged more quickly in this case than the run with no IM. In this worst-case scenario, then, PBIM still provides value by facilitating a trade-off between preventing reward hacking and increasing training efficiency.

\subsubsection{Discussion: GRM D = 1}\label{sec:d1explanation}

When compared with PBIM, these runs across all environments consistently converge less quickly, but to a more consistently optimal policy. This is to be expected, when analyzing the difference between these methods and PBIM from a reward-horizon-extending viewpoint following \cite{laud2004theory}. While PBIM represents one end of the spectrum of GRM methods parameterized by Equation~\eqref{hyperparam_tuning}, the one in which compensatory adjustments to the intrinsic motivation are furthest-removed in time, these methods represent the other end, in which the adjustments to preserve optimality happen only a single time step after the reward for which they are adjusting. Thus, while these methods are more at risk of being quickly discovered by the agent to be providing essentially useless IM, they are also less likely than PBIM to run the risk of ``fooling'' the agent too thoroughly, to the point where it doesn't quite converge to the optimal policy in a timely manner.

Another point of note is that, in all three hyperparameter settings tested here, one of the $D=1$ runs performed substantially better than the other, but that in the $\gamma = .99$ case, which run performed well (normalized vs. non-normalized) switched from the other two cases. In the remainder of this section, we offer an explanation for why this occurred.

In the $D=1$ non-normalized case, for all timesteps but the first and last, $F' = F_t - \frac{F_{t-1}}{\gamma}$. Because there is only a single timestep between the intrinsic reward being accounted for and that which is being obtained by the agent, there is unlikely to be as much of a difference in these terms as in the other methods, as they will be obtained in the same region of the state space. As such, both the $D = 1$ normalized and non-normalized runs will be dominated by terms in which we can approximate $F_t \approx F_{t-1} = F_{approx}$, or $\Delta F = F_t - F_{t-1} \approx 0$. Interpreting our results in light of this approximation, they become more clear.

Under this approximation, the non-normalized reward becomes
\begin{align}
    F' \approx -\frac{1-\gamma}{\gamma}F_{approx}
\end{align}
for most timesteps, except for when going from a high-IM region to a low-IM region. This reward, for a consistently positive IM, is consistently negative, and more negative in the higher-IM regions of the space. However, the actual immediate reward for transitioning from a low- to high-IM region of the environment is still positive. Because of this, in high-$\gamma$ environments, we should expect the non-normalized version of this method to do poorly, as it will be biased away from existing for long periods of time in high-IM regions of the state space. However, it should perform better in environments with a low $\gamma$, as this bias becomes outweighed by the immediate bias towards the immediate reward obtained by traversing from a low-IM to high-IM region of the state space. Those results are precisely what we see here. Also of note, for the non-normalized version, is that it performs better with high $\alpha$ than with lower $\alpha$: this can be explained by the IM reward signal being stronger in with higher $\alpha$ values.

For the $D=1$ normalized case, in our approximation, the reward becomes
\begin{align}
    F' \approx \frac{1-\gamma}{\gamma}\left(\bar{F} - F_{approx} \right).
\end{align}

This significantly mitigates the problems faced by the non-normalized version by cutting down the bias against high-IM regions by an order of magnitude, allowing the relatively infrequent cases where $\Delta F \neq 0$ to more easily dominate. This explains why the normalized version of GRM $D=1$ performs (relatively) well in the $\gamma = .995$ environments where the non-normalized version fails. However, this comes at a cost, one that becomes more significant in low-$\gamma$ environments: if it finds itself in a part of the state space where $F_{approx} << \bar{F}$, it can obtain immediate positive rewards by staying in that low-IM portion of the state space. This creates a bias against exploration when in the middle of low-IM regions in low-horizon environments, where the $\Delta F \neq 0$ to be sought might be several timesteps away. This bias has to be unlearned by the agent, resulting in convergence even later than the no-IM case in the $\gamma = .99$ environment.

Of course, we know from Theorem \ref{PBGRM_theorem} that these biases will theoretically be unlearned, given enough time, but they help explain the sometimes-unreliable practical performance of these methods in this particular environment, and demonstrate that GRM functions should ideally be chosen to best fit a particular environment, and that this is not always a trivial task.

\subsubsection{Discussion: GRM D = 10}

Of our three methods, this one was the most consistently high-performing throughout different hyperparameters in this environment, both in its normalized and non-normalized formulations. In fact, it is the only one of the three versions of our method we tested for which normalization didn't have a large effect on its performance in any of our runs. It seems to exist in a ``sweet spot'' between PBIM and GRM $D=1$, wherein it doesn't shift the reward bias so much as to be susceptible to the issues with PBIM discussed in Section~\ref{converting}, but separates the intrinsic being obtained by the agent and that being accounted for enough so as to be robust to the issues discussed in Section~\ref{sec:d1explanation}, as well. As such, both versions were able to converge or nearly converge more quickly than any other method tested in every configuration of this environment tested, including IM, while also converging to a statistical tie for the most optimal policy. We take this to be straightforward and compelling evidence for the efficacy of GRM at both speeding training and preventing convergence to a suboptimal policy.

\subsection{Cliff Walking}

\label{sec:cw}

We also tested our method in cliff walking~\cite{sutton2018reinforcement}, a classic reinforcement learning task in which an agent is directed to find a goal state at the end of a horizontally long grid world, where the bottom row represents ``the cliff'': a set of stats that must be avoided. The environment used is depicted in Figure~\ref{fig:cw}. In this environment, the agent starts in the bottom left tile, marked with an S, and must reach the bottom right tile, marked with a G. All other tiles in the bottom row are ``cliffs.'' At every time step, the agent can move either up, down, left, or right. Entering cliff tiles returns a reward of -100, and reaching the goal tile instead returns 100. Every other action has a reward of -1. The episode ends when the agent reaches a cliff or goal tile, or when the maximum episode length~(50 steps) is reached. The problem described in \citep{sutton2018reinforcement} includes a slip chance: a probability that the agent will move downward instead of its intended direction to introduce stochasticity, safety, and hunger of risk. We opt to use no chance of slipping in this version of the environment to focus on only the exploration aspect.
\begin{figure}
    \centering
    \includegraphics[width=\linewidth]{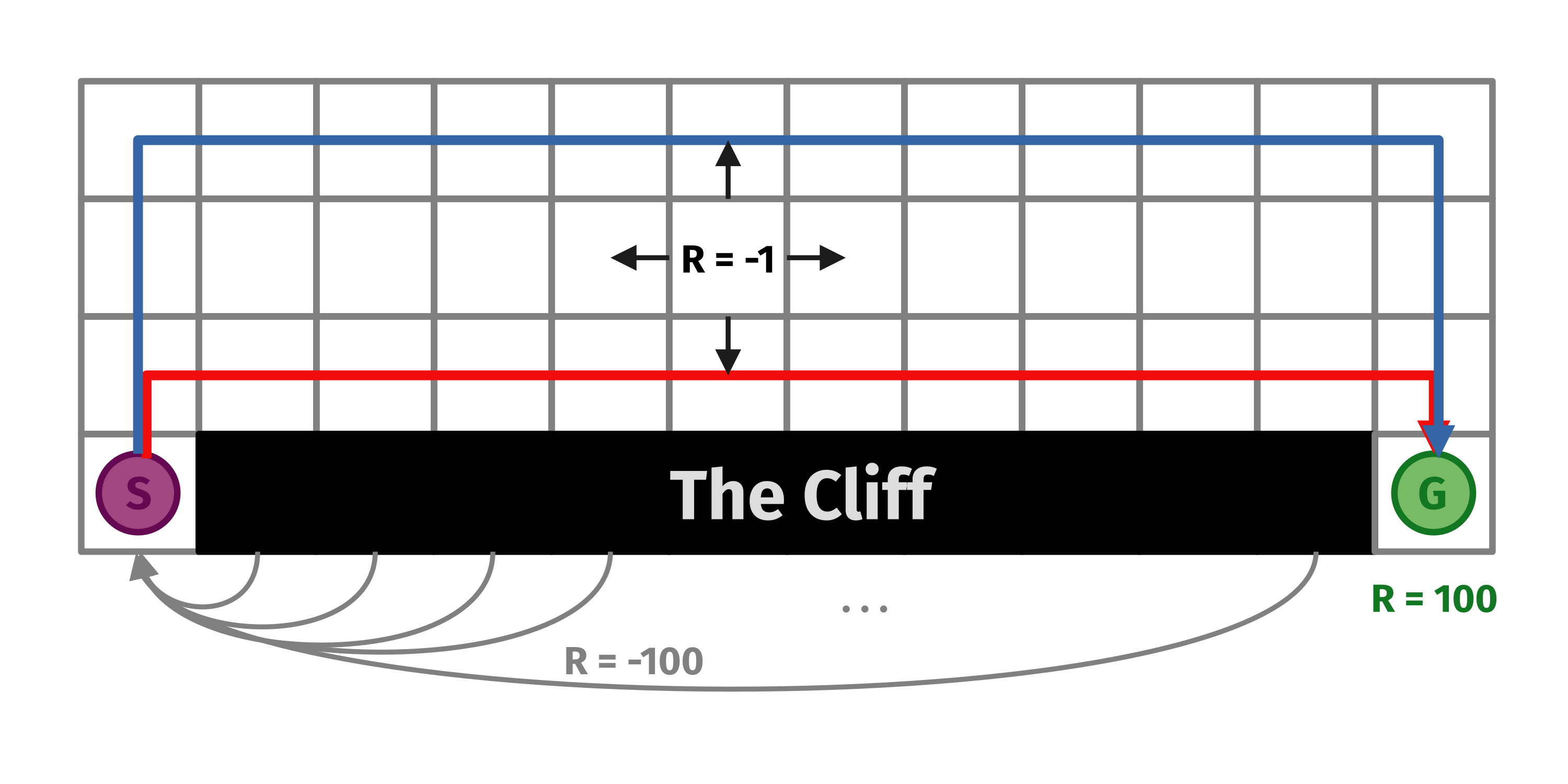}
    \caption{Cliff walking scenario (adapted from~\cite{sutton2018reinforcement}).The red arrow shows the optimal path to take when no slip chance is present, while the blue arrow shows a ``cautious,'' suboptimal path that the agent may take, particularly if it is trying to maximize some IM term in centivizing exploration.}
    \label{fig:cw}
\end{figure}

For this experiment, we trained a simple, tabular Q-learning agent with eight types of intrinsic motivation (similarly to the experiments in section~\ref{minigrid}): none, a baseline IM, non-normalized PBIM, normalized PBIM, and four GRM variants: $D=1$ and $D=10$ with and without normalization. For our baseline IM, however, rather than use the simple tabular IM from Section \ref{minigrid}, we used a Random Network Distillation (RND), a SOTA exploration IM \cite{burda2018exploration} with dependence on full training trajectories, rather than simply on an episode's history.

This experiment has a different focus than in Section~\ref{sec:empirical}:
\begin{itemize}
\item
1. We use a Q-learning agent, rather than PPO, to show that our method's effectiveness does not generally depend on the optimization algorithm used.
\item
2. We use a simple environment, where the optimal policy is trivial to determine, and easy to depict graphically. It should thus be easily apparent if an agent's policy diverges from what is optimal.
\item
3. Because this is a simpler problem, the agent here does not require nor actually benefit from the intrinsic reward. Instead, we are interested in finding out if the intrinsic reward hinders policy convergence. Put another way: we want to see if PBIM and/or GRM can prevent an agent from being distracted by an unhelpful IM term.
 \end{itemize}
 

We trained the Q-learning agents for 5000 episodes with $\gamma = 0.99$, using $epsilon$-greedy exploration with $\epsilon = 1.0$ decreasing by $5 \cdot 10^{-3}$ after every episode down to a minimum of $0.1$. Due to the small observation space, we used an RND predictor network learning rate of $10^{-6}$\footnote{A higher learning rate here would have run the risk of the RND reward later on in training becoming zero or effectively zero, trivially conserving the optimal policy and thus defeating the purpose of this experiment.} and an intrinsic reward scaling coefficient of $1000$.

\subsubsection{Discussion}
Figure~\ref{small_cliffwalker} shows each agent's training performance. As the goal is to reach the goal state as quickly as possible, episode length here is a proxy for the performance of an agent. We expect agents to start with short episodes before they learn not to fall off the cliff, then increase the episode length as they explore the environment, and finally decrease the episode length again as they reinforce the optimal route. An important point to note about this figure, as well as Figure~\ref{fig:large_cliffwalker}, is that these are episode returns \textit{during training}. This means that, due to the nonzero chance $\epsilon$ of taking a random action, policies that consistently take a suboptimal path toward the goal state will actually have higher \textit{training} loss than the optimal policy. This explains why the methods that obtain higher returns in Figure \ref{fig:returnsmall} have higher final episode lengths in Figure~\ref{fig:lengthsmall}.


\begin{figure}
    \centering
    \begin{subfigure}{\columnwidth}
        \includegraphics[width=\columnwidth]{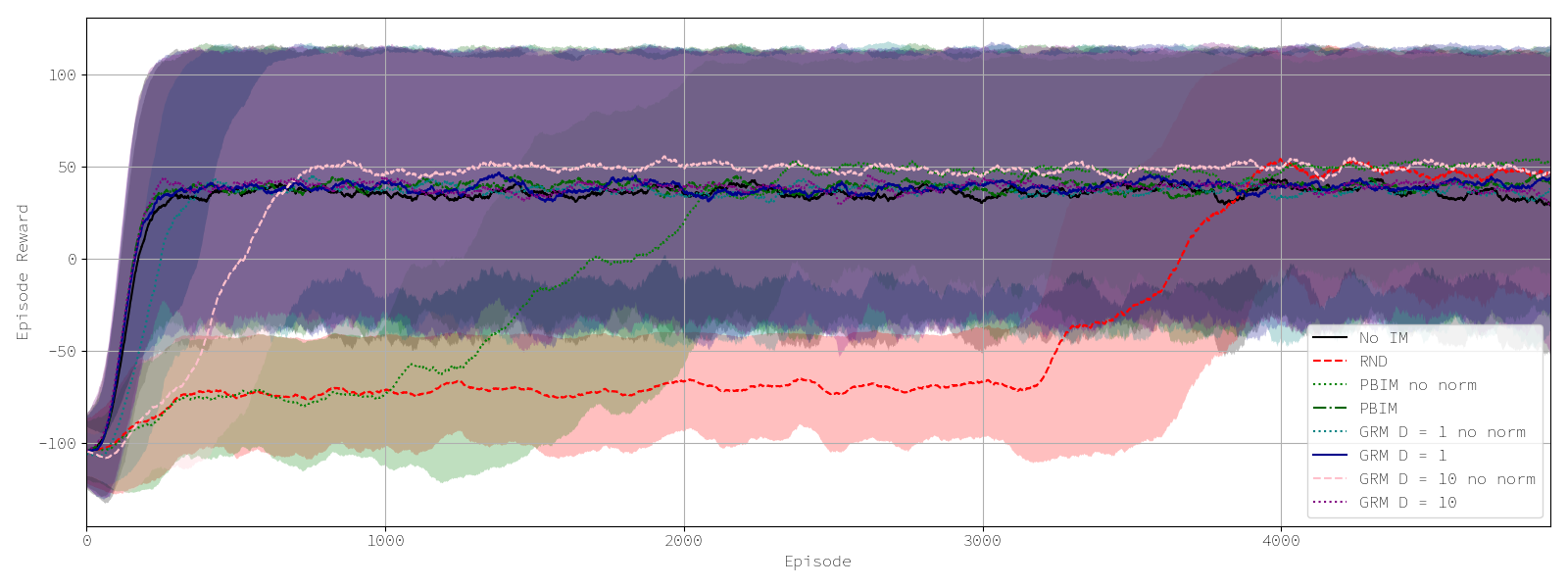}
        \caption{Average Episode Return}
        \label{fig:returnsmall}
    \end{subfigure}
    
    \begin{subfigure}{\columnwidth}
        \includegraphics[width=\columnwidth]{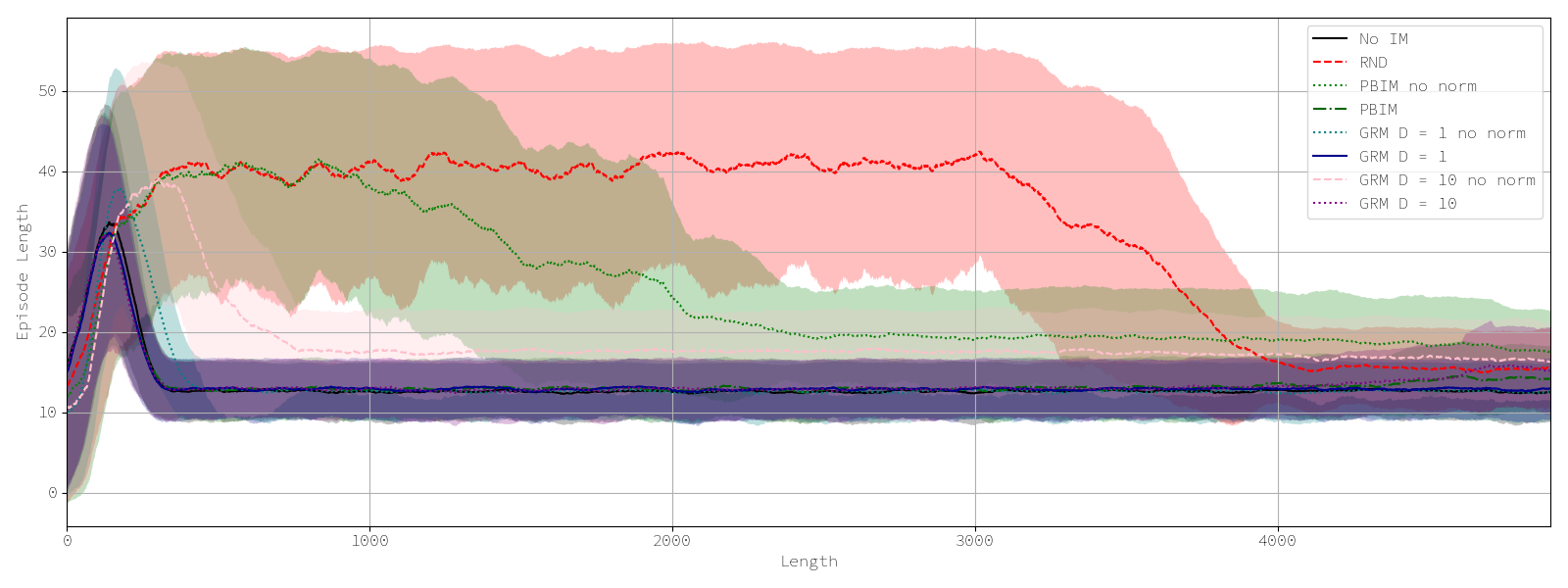}
        \caption{Episode Length}
        \label{fig:lengthsmall}
    \end{subfigure}
    \caption{Average cumulative extrinsic return and episode length for the cliff walking environment. Error bars are standard deviations over 10 runs.}
    \label{small_cliffwalker}
\end{figure}

Here, the vanilla Q-Learning agent with no intrinsic motivation is the fastest to converge, finding the optimal policy in around 500 episodes. It is important to highlight that Q-Learning reaches the optimal policy (Later demonstrated in Figure~\ref{fig:noim}). The three normalized methods (PBIM, GRM $D = 1$, and GRM $D = 10$) converge to the optimal policy at a very similar rate, as does normalized GRM with $D = 10$. 

On the other hand, the other IM methods are ``distracted'' by the intrinsic reward and converge to a suboptimal policy at a slower rate. The slowest to converge is vanilla RND exploration, only obtaining rewards similar to the other policies at around 4,000 episodes. This is a textbook example of the noisy TV problem, in which the agent is distracted from the task at hand due to the intrinsic reward. Non-normalized PBIM and GRM suffer from a similar issue, though converge much faster than base RND. GRM with $D = 1$ is only slightly distracted by the intrinsic reward and reaches optimality reasonably well. These results seem to suggest that normalization is key in this environment, and that its importance increases as the parameter $D$ is scaled up.

An analogous trend can be derived from Figure~\ref{fig:lengthsmall}: Q-Learning and the normalized IM methods quickly explore the environment before settling on the optimal policy (Which takes a total of 13 actions to execute). The non-normalized IM methods explore the environment for much longer, but similarly, take a longer time to reach convergence.

To better show the difference between suboptimal and optimal policies in a testing environment, the final average test return (in an environment with $\epsilon=0$) and episode length of each policy is depicted in Figure~\ref{fig:returnsmall_test}.

\begin{figure}
    \centering
    \begin{subfigure}{\columnwidth}
        \includegraphics[width=\columnwidth]{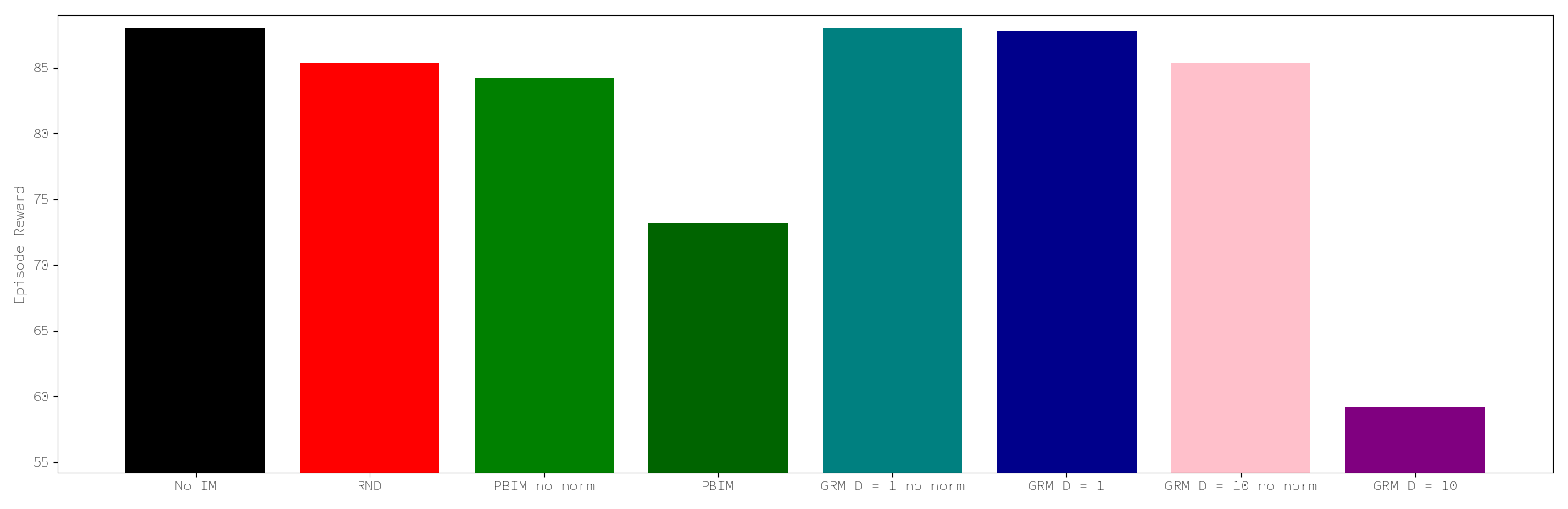}
        \caption{Average Episode Return}
        \label{fig:returnsmall_test}
    \end{subfigure}
    
    \begin{subfigure}{\columnwidth}
        \includegraphics[width=\columnwidth]{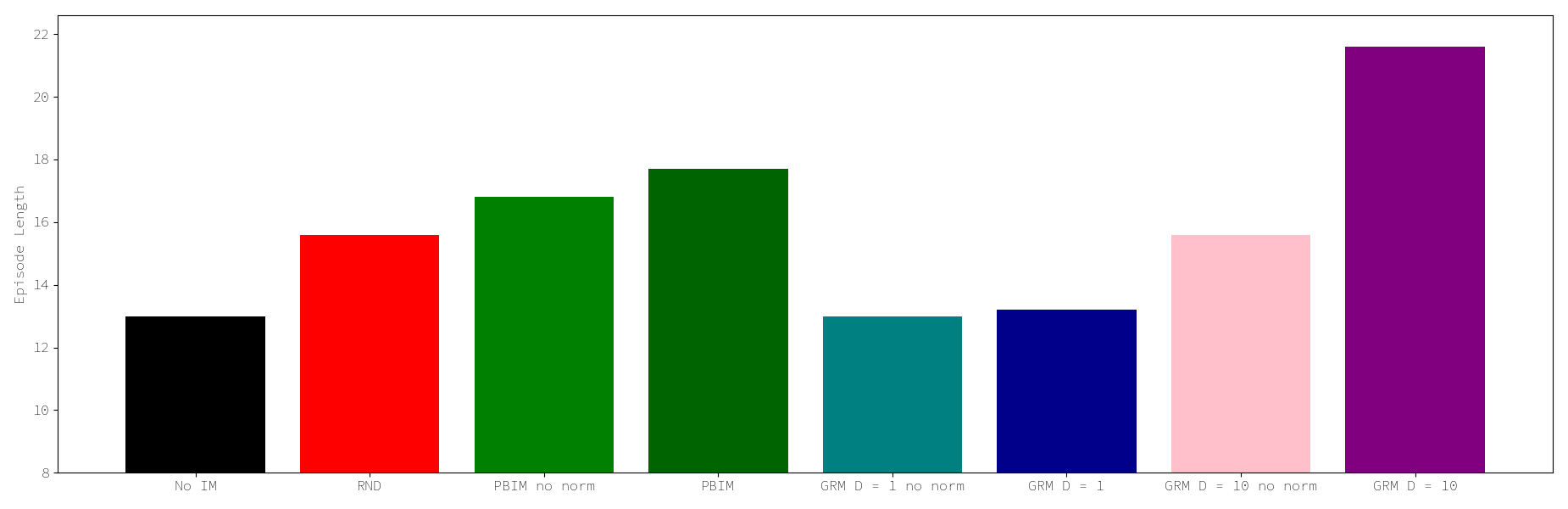}
        \caption{Episode Length}
        \label{fig:lengthsmall_test}
    \end{subfigure}
    \caption{Average cumulative extrinsic return and episode length for the final policies in the cliff walking environment. Results are averaged over 10 runs per trained agent. Note the plots are shifted on the $y$-axis.}
    \label{small_cliffwalker_test}
\end{figure}

Figure~\ref{small_cliffwalker_test} shows the average cumulative reward and episode length of the eight trained agents, averaged over the 10 runs. Three agents (No IM Q-Learning, and the two GRM with $D = 1$) consistently reach the optimal policy with a total reward of 88 and an episode length of 13. The other methods, most notably non-normalized GRM with $D = 10$ take longer to reach the goal, thus obtaining a lower reward.


\begin{figure}[t]
    \begin{subfigure}{0.49\columnwidth}
        \centering
        \includegraphics[width=\columnwidth]{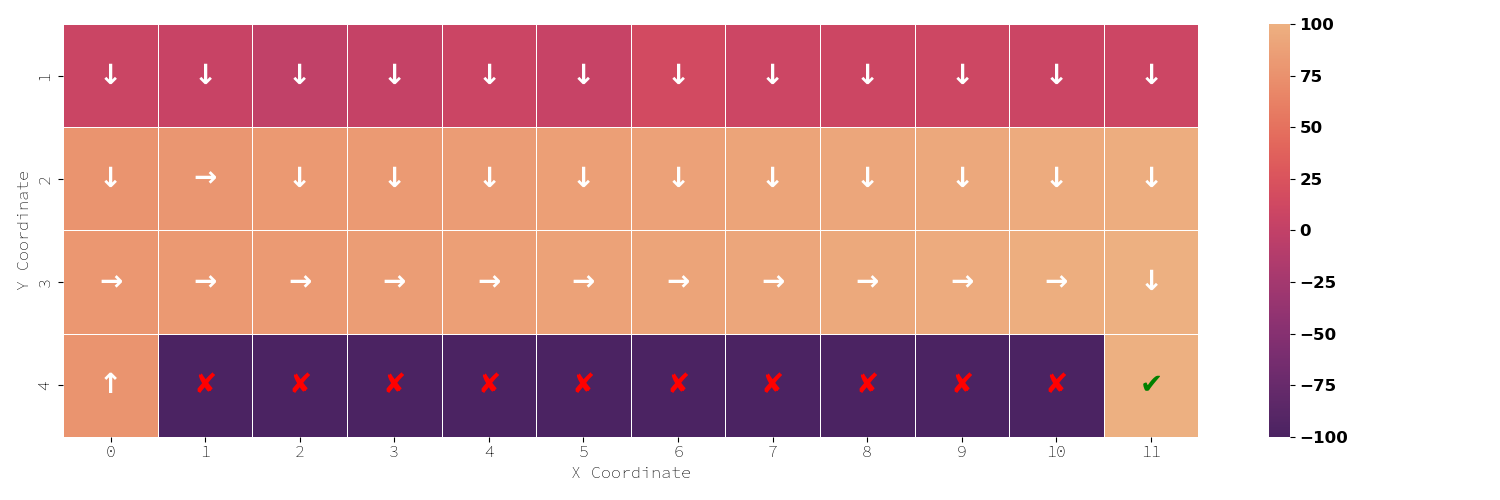}
        \caption{No IM}
        \label{fig:noim}
    \end{subfigure}
    \begin{subfigure}{0.49\columnwidth}
        \centering
        \includegraphics[width=\columnwidth]{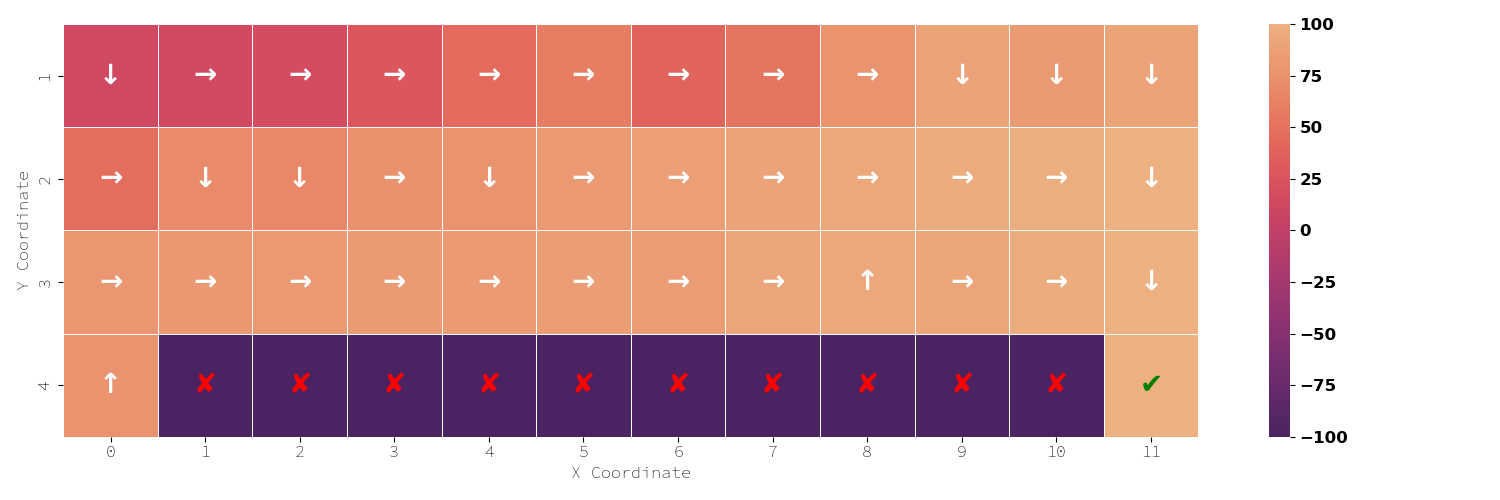}
        \caption{RND}
        \label{fig:rnd}
    \end{subfigure}
    \begin{subfigure}{0.49\columnwidth}
        \centering
        \includegraphics[width=\columnwidth]{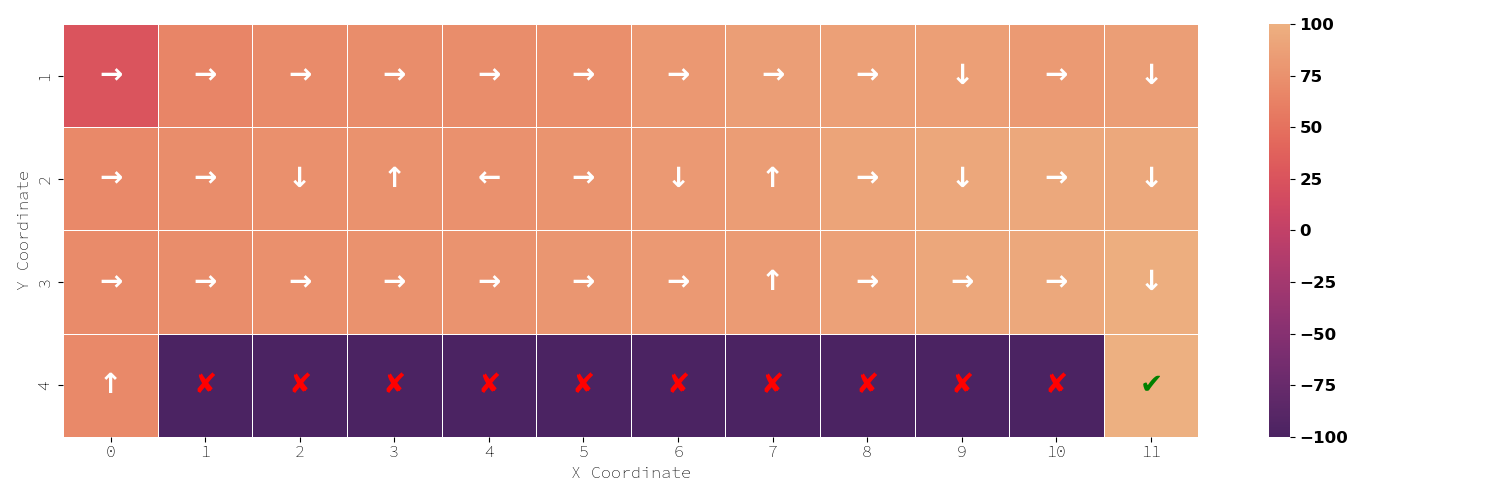}
        \caption{PBIM, not normalized}
        \label{fig:pbimnonorm}
    \end{subfigure}
    \begin{subfigure}{0.49\columnwidth}
        \centering
        \includegraphics[width=\columnwidth]{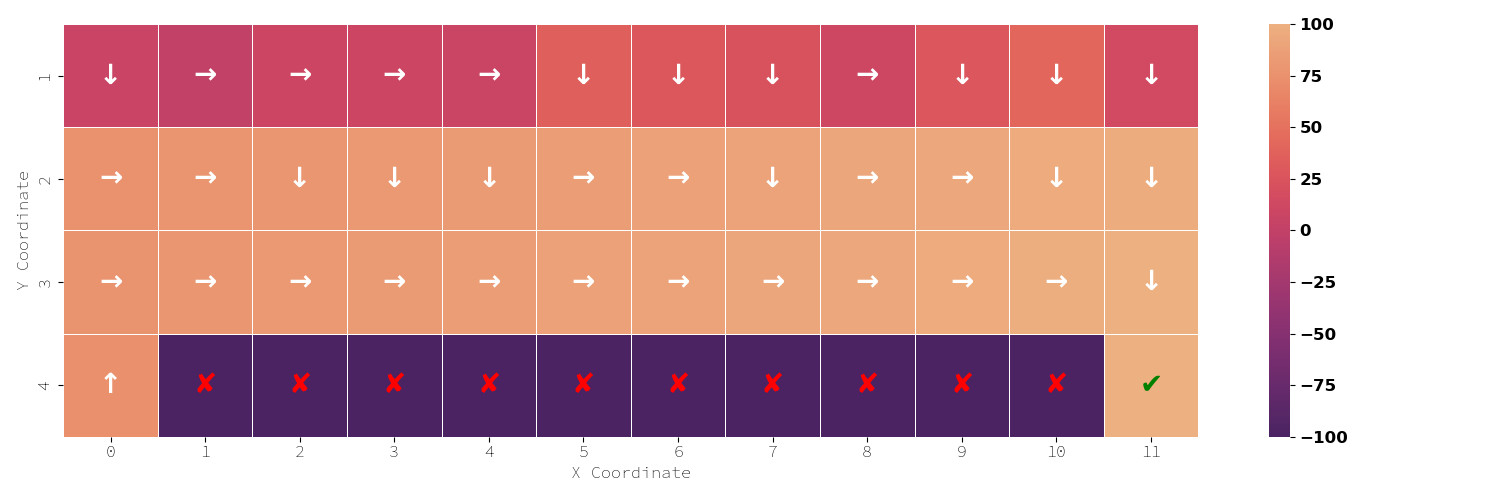}
        \caption{PBIM}
        \label{fig:pbim}
    \end{subfigure}
    \begin{subfigure}{0.49\columnwidth}
        \centering
        \includegraphics[width=\columnwidth]{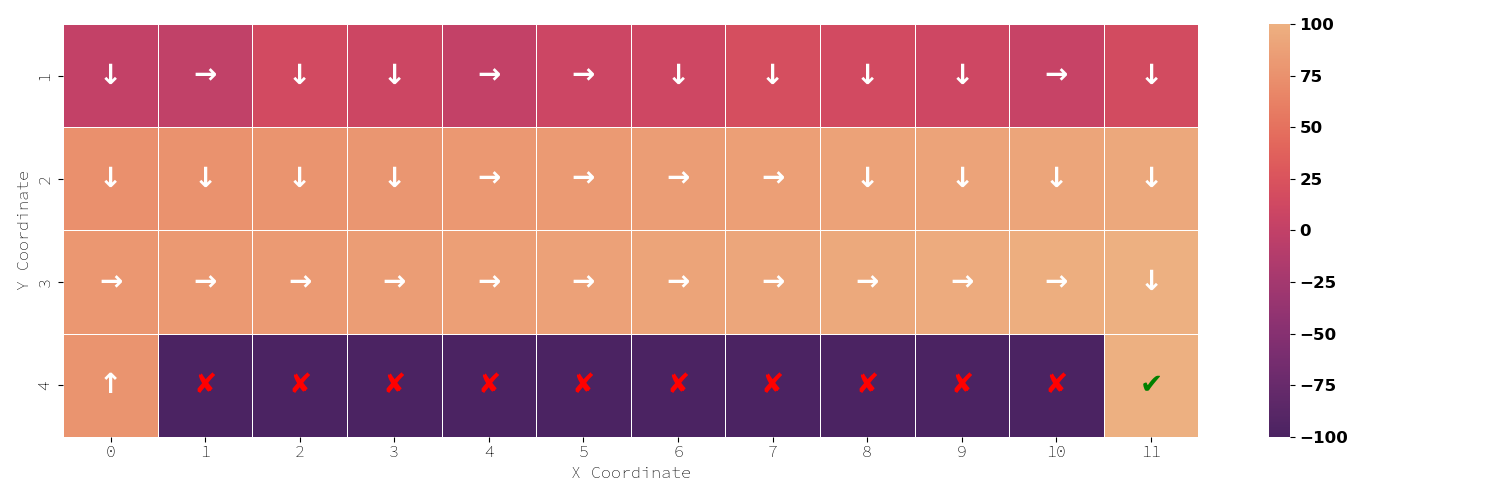}
        \caption{GRM $D = 1$, not normalized}
        \label{fig:grm1nonorm}
    \end{subfigure}
    \begin{subfigure}{0.49\columnwidth}
        \centering
        \includegraphics[width=\columnwidth]{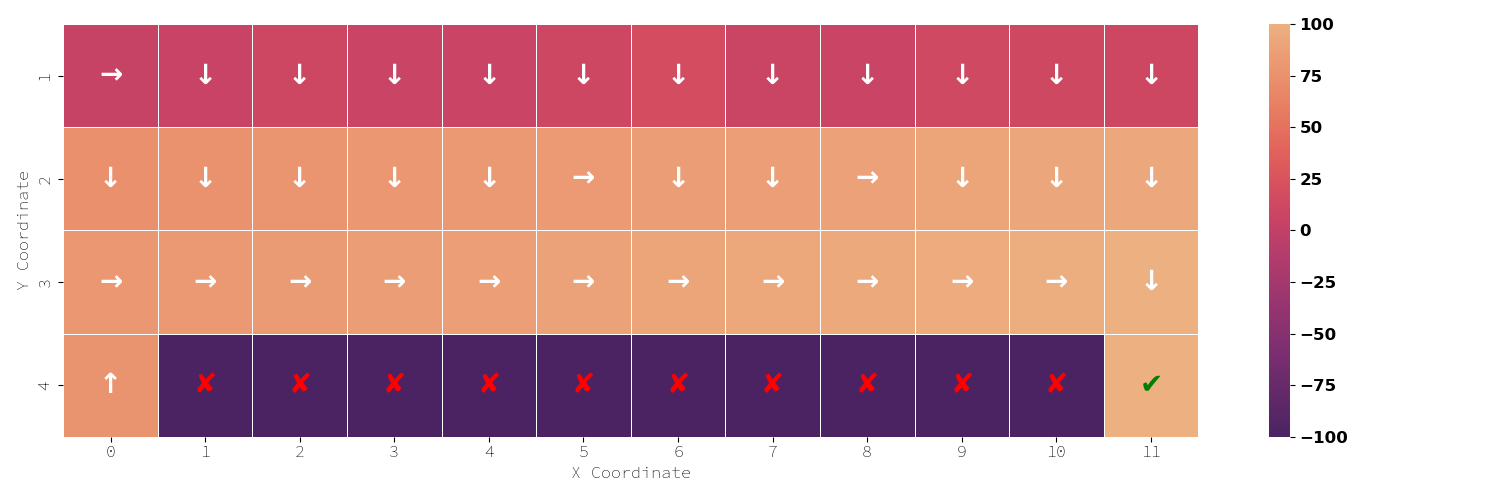}
        \caption{GRM $D = 1$}
        \label{fig:grm1}
    \end{subfigure}
    \begin{subfigure}{0.49\columnwidth}
        \centering
        \includegraphics[width=\columnwidth]{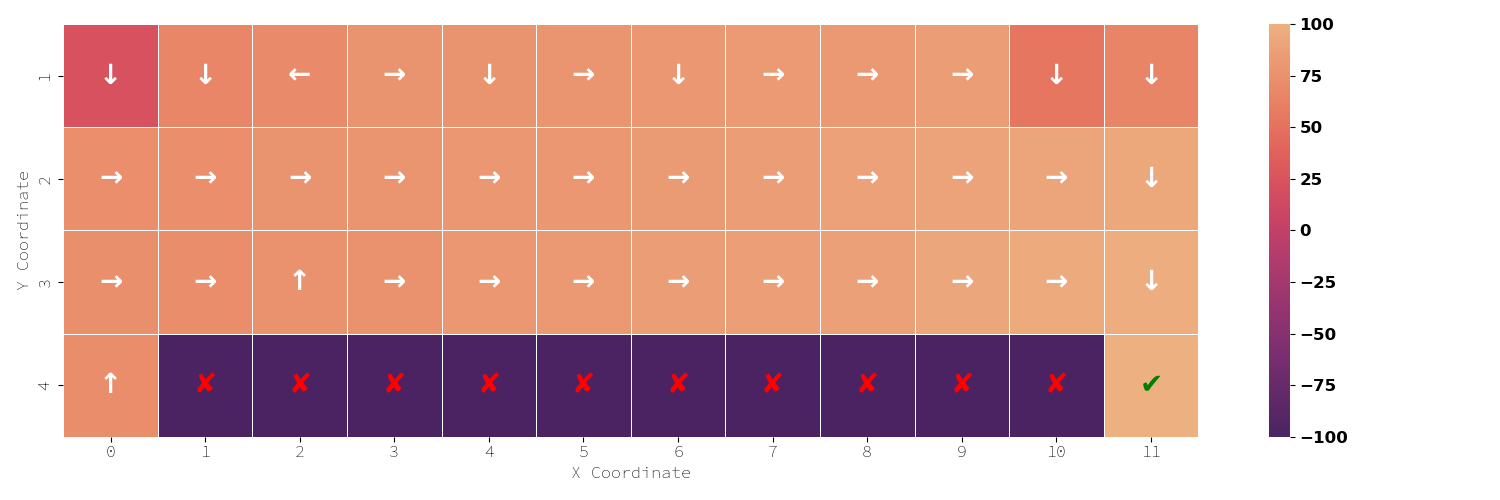}
        \caption{GRM $D = 10$, not normalized}
        \label{fig:grm10nonorm}
    \end{subfigure}
    \begin{subfigure}{0.49\columnwidth}
        \centering
        \includegraphics[width=\columnwidth]{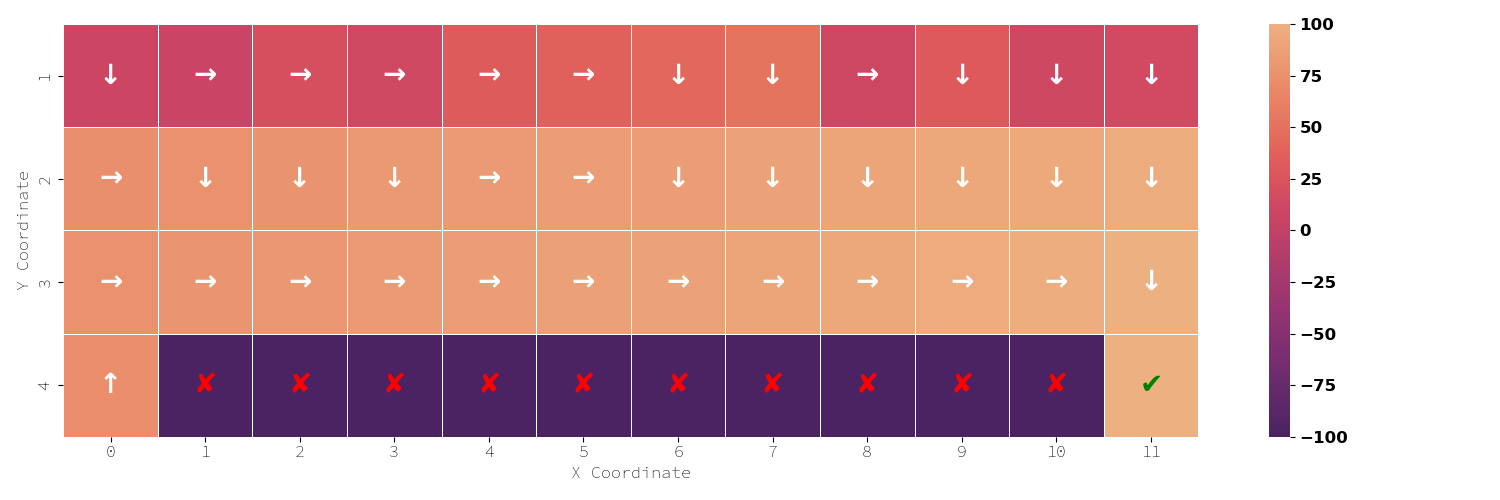}
        \caption{GRM $D = 10$}
        \label{fig:grm10}
    \end{subfigure}
    \caption{Final policies of trained agents and their estimated Q-values on the cliff walking environment. Arrows indicate the action with the highest estimated Q-value in each position. A brighter hue indicates a higher Q-value.}
    \label{Q_values}
\end{figure}

We can confirm the above statements by looking at the final policies, shown in Figure~\ref{Q_values}. The figure contains, for each IM technique, the most common action taken by the 10 trained versions of that model and the estimated Q-value of said action. Here we observe that the optimal policy is reached by Q-Learning, PBIM, both versions of GRM with $D = 1$, and normalized GRM with $D = 10$. Note that this Figure only contains the action most selected across the 10 policies, and does not account for mistakes in each individual policy.

\subsection{Long Cliff Walking}
\label{sec:lw}

The environment in Section~\ref{sec:cw} was simple enough that IM was not necessary to find a solution. To test our method's efficacy in more dauntingly sparse-reward environments that require reward shaping to solve effectively, we evaluate PBIM in a modified version of cliff walking where the goal is further away as the grid is horizontally expanded. This version features a $4 \times 50$ grid, where the start and goal are on the leftmost and rightmost tiles in the bottom row, and all other bottom tiles are cliffs.

We used the same experimental setup as the cliff walking environment, including the eight types of agents, but instead increased the number of episodes to 10,000 and the maximum steps per episode to 100. The agents were trained with $\gamma = 0.99$, $\epsilon = 1.0$ decreasing by $5 \cdot 10^{-4}$ after every episode down to a minimum of $0.1$. We used the same RND configuration: a learning rate of $10^{-6}$ and reward scaling of $1000$. The experiment is run 10 times per agent.

Compared to the regular cliff walker, the optimal policy is still plainly observable. However, the reward signal is much sparser, and traditional $\epsilon$-greedy exploration will have trouble finding the reward on the far right of the grid. We thus expect to see a much more successful performance by IM methods. In this scenario, the maximum total reward obtainable by the agent is 49, and it takes 52 steps to reach the goal.

Figure~\ref{fig:large_cliffwalker} shows the training performance for the trained agents, and Figure~\ref{fig:large_cliffwalker_test} shows the final performance with no $\epsilon$-greedy exploration. In Figure~\ref{fig:returnlarge}~we can observe that, in the long run, GRM with $D = 1$ reaches the highest reward during training, followed by Q-Learning with no IM, then PBIM with no normalization and RND. GRM with a much higher delay cannot, nor PBIM, reach the training reward of the other agents. We however observe in Figure~\ref{fig:lengthlarge} that, in terms of episode length, these ``worse-performing'' techniques achieve the shortest episode length, which could signal a closer convergence to the optimal policy as we studied in Section~\ref{sec:cw}.

Figure~\ref{fig:returnlarge_test} shows that No-IM, both GRM with $D = 1$, normalized GRM with $D = 10$, and normalized PBIM reap the highest reward on average, which hovers around 48 (1 less on average than the theoretical average). In Figure~\ref{fig:lengthlarge_test} we can similarly observe these models require the lesser number of steps to reach the goal. Still, these differences are small and the other models require at most 1 more step on average to reach the goal.

\begin{figure}
    \centering
    \begin{subfigure}{\columnwidth}
        \includegraphics[width=\columnwidth]{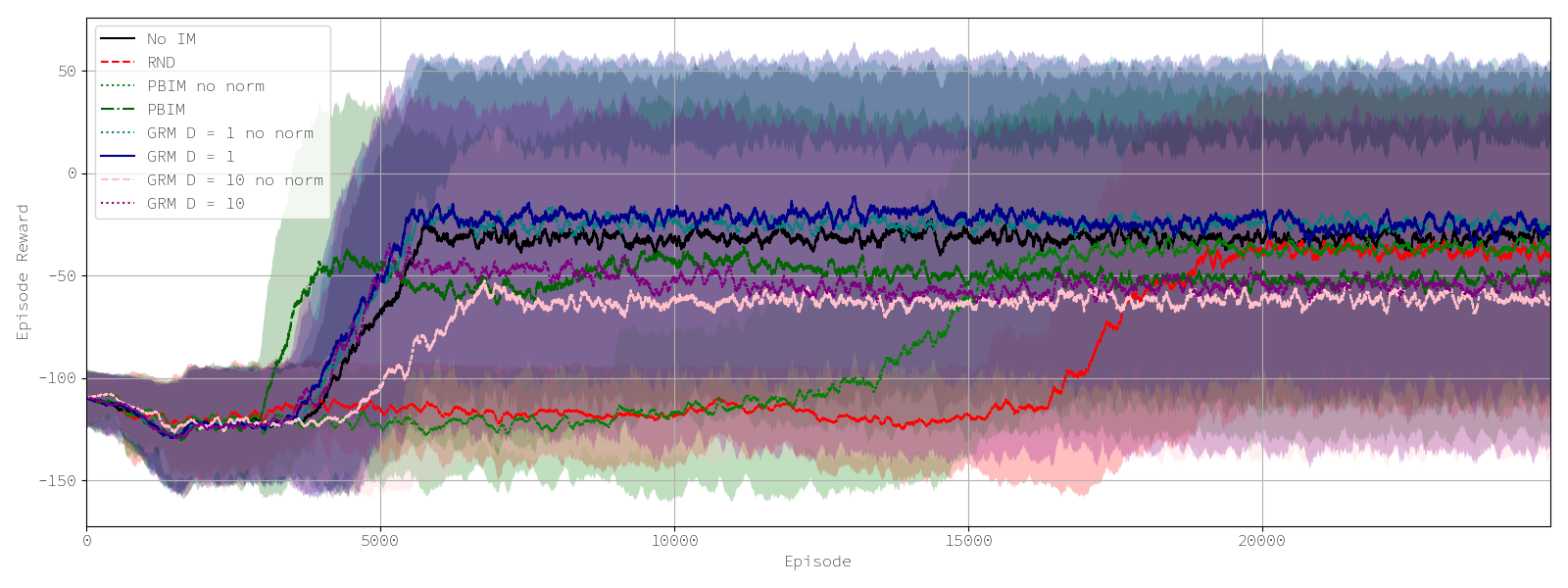}
        \caption{Average Episode Return}
        \label{fig:returnlarge}
    \end{subfigure}
    
    \begin{subfigure}{\columnwidth}
        \includegraphics[width=\columnwidth]{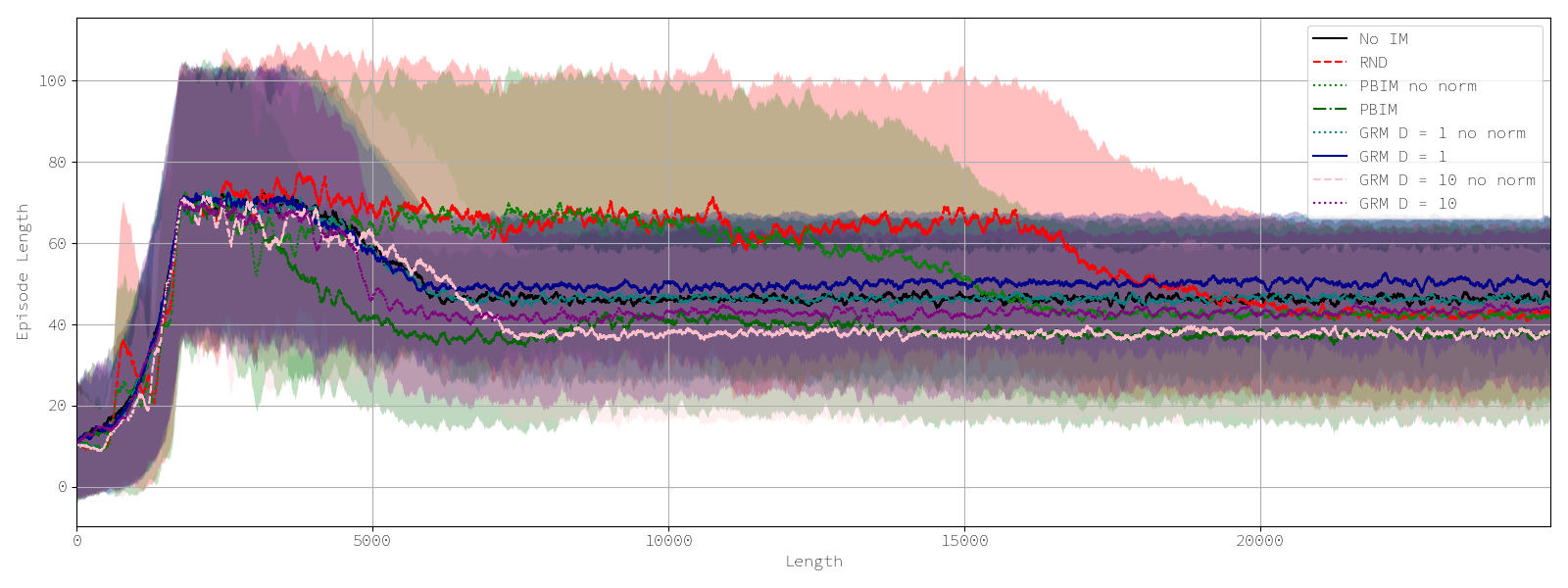}
        \caption{Episode Length}
        \label{fig:lengthlarge}
    \end{subfigure}
    \caption{Average cumulative extrinsic return and episode length for the long cliff walking environment. Error bars are standard deviations over 10 runs.}
    \label{fig:large_cliffwalker}
\end{figure}

\begin{figure}
    \centering
    \begin{subfigure}{\columnwidth}
        \includegraphics[width=\columnwidth]{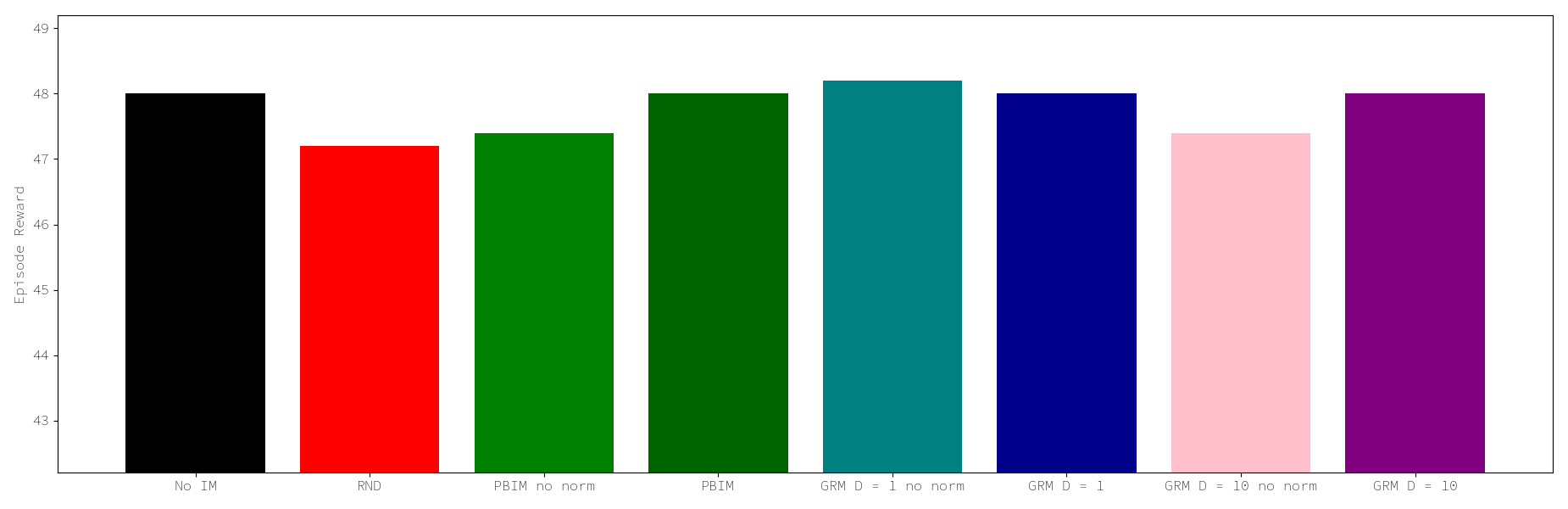}
        \caption{Average Episode Return}
        \label{fig:returnlarge_test}
    \end{subfigure}
    
    \begin{subfigure}{\columnwidth}
        \includegraphics[width=\columnwidth]{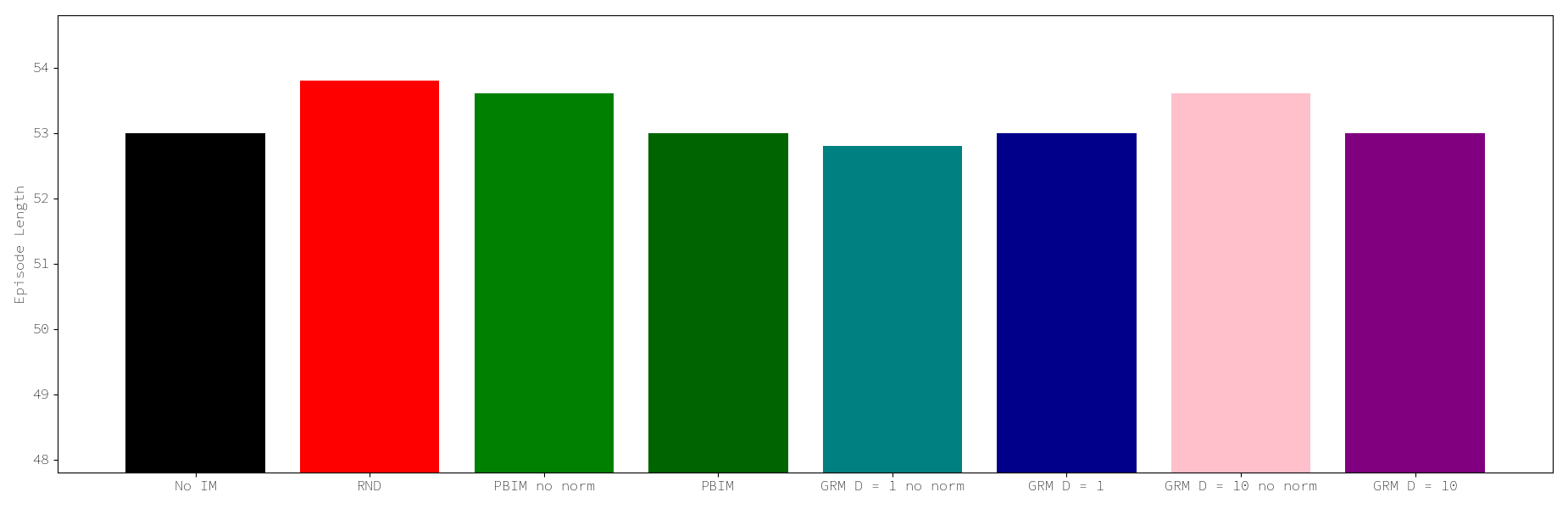}
        \caption{Episode Length}
        \label{fig:lengthlarge_test}
    \end{subfigure}
    \caption{Average cumulative extrinsic return and episode length for the final policies in the long cliff walking environment. Results are averaged over 10 runs per trained agent.}
    \label{fig:large_cliffwalker_test}
\end{figure}

Figure~\ref{large-Q_values} shows the policies for all trained agents, built picking the most frequently picked action across the 10 runs. Two policies reach optimality: PBIM and GRM with $D = 10$ and no normalization. All other models have slight inefficiencies that cause them to go into the middle path at times.

\begin{figure}[t]
    \begin{subfigure}{0.49\columnwidth}
        \centering
        \includegraphics[width=\columnwidth]{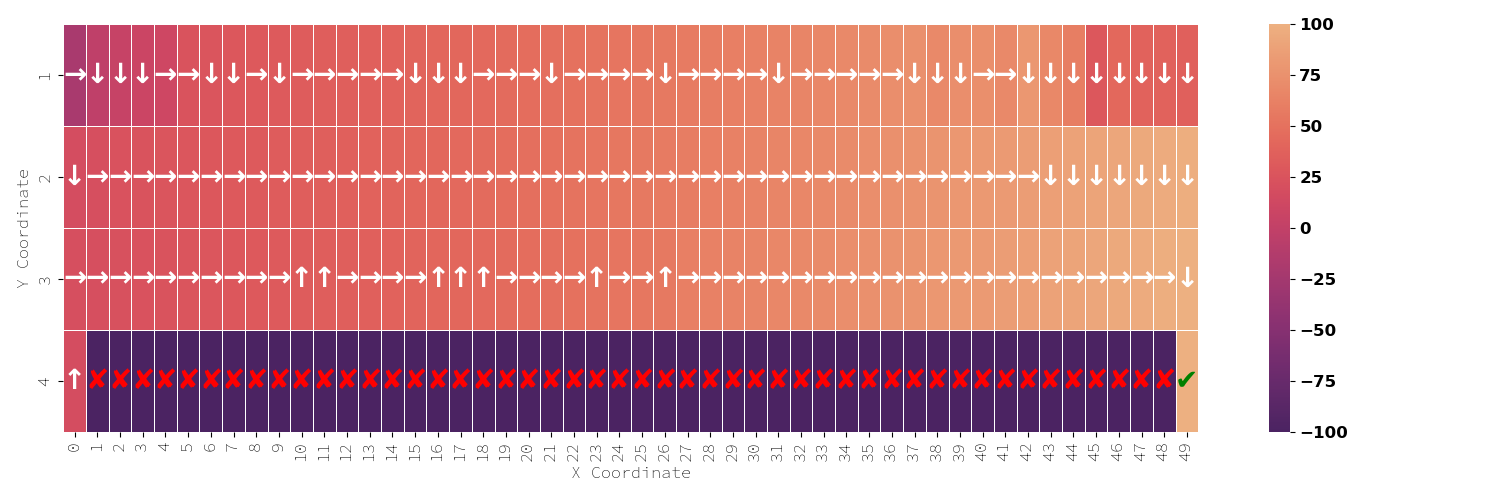}
        \caption{No IM}
        \label{fig:large-noim}
    \end{subfigure}
    \begin{subfigure}{0.49\columnwidth}
        \centering
        \includegraphics[width=\columnwidth]{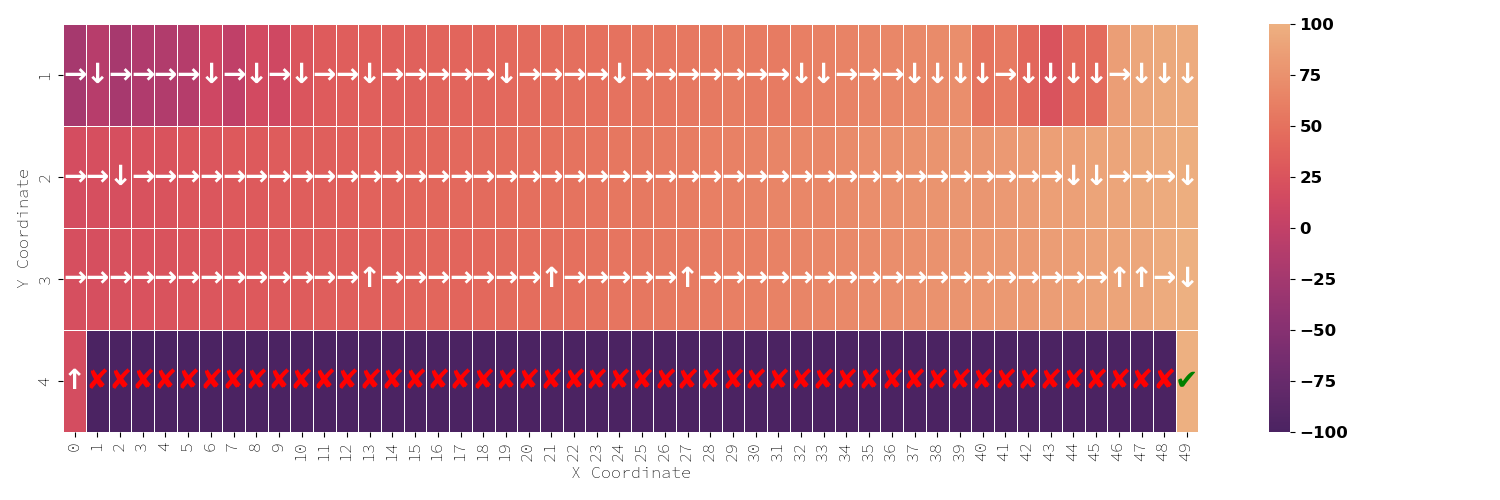}
        \caption{RND}
        \label{fig:large-rnd}
    \end{subfigure}
    \begin{subfigure}{0.49\columnwidth}
        \centering
        \includegraphics[width=\columnwidth]{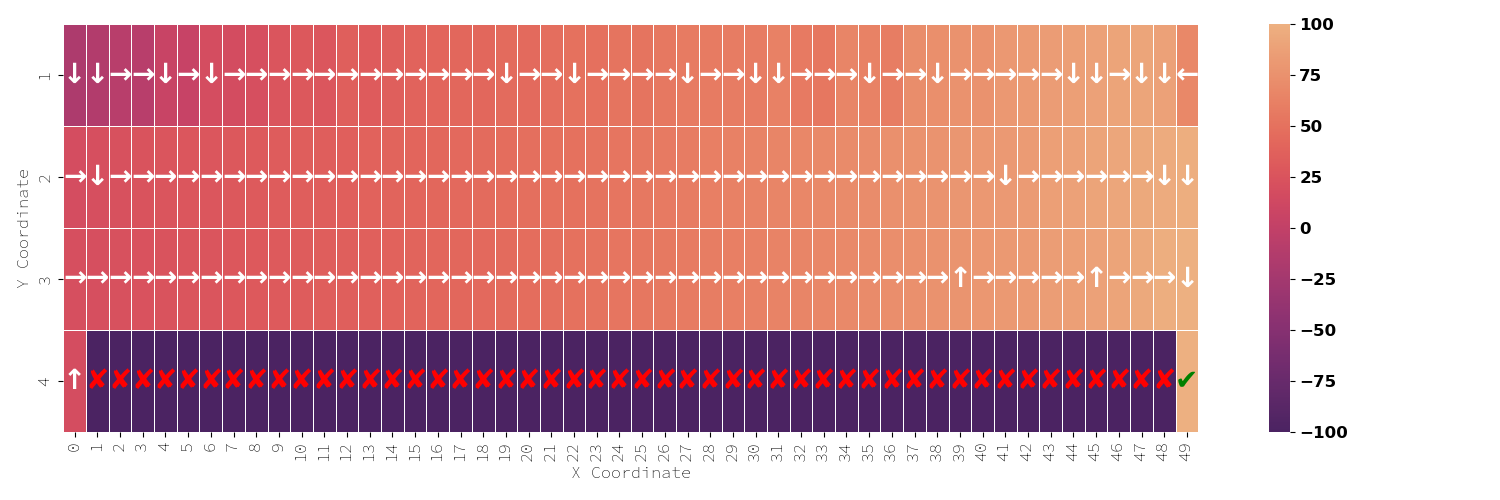}
        \caption{PBIM, not normalized}
        \label{fig:large-pbimnonorm}
    \end{subfigure}
    \begin{subfigure}{0.49\columnwidth}
        \centering
        \includegraphics[width=\columnwidth]{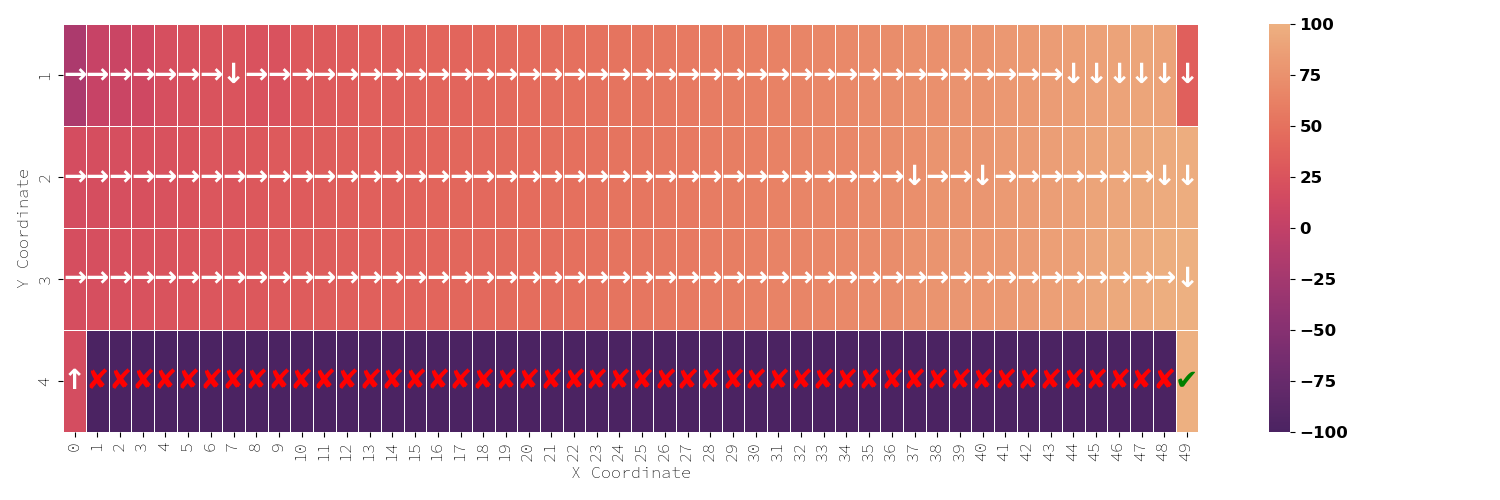}
        \caption{PBIM}
        \label{fig:large-pbim}
    \end{subfigure}
    \begin{subfigure}{0.49\columnwidth}
        \centering
        \includegraphics[width=\columnwidth]{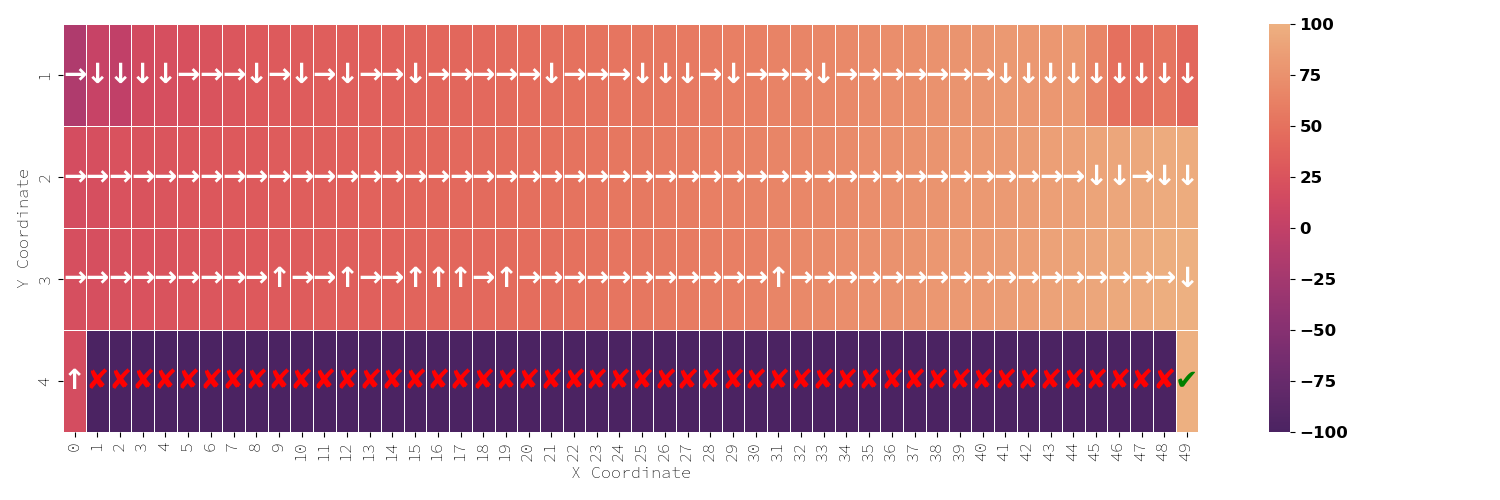}
        \caption{GRM $D = 1$, not normalized}
        \label{fig:large-grm1nonorm}
    \end{subfigure}
    \begin{subfigure}{0.49\columnwidth}
        \centering
        \includegraphics[width=\columnwidth]{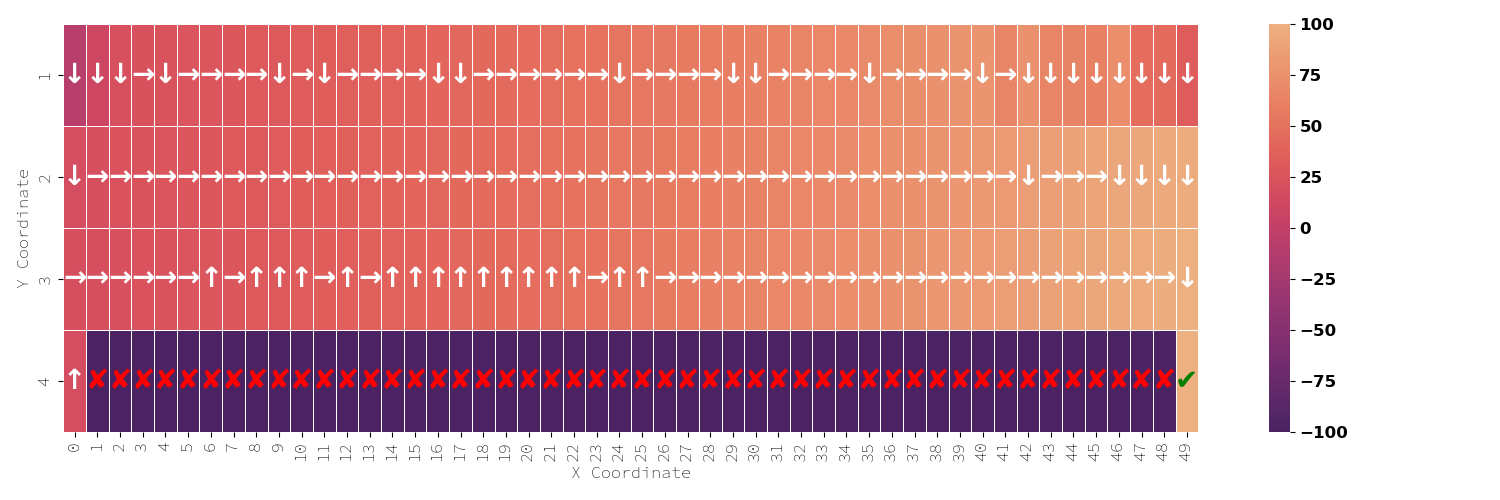}
        \caption{GRM $D = 1$}
        \label{fig:large-grm1}
    \end{subfigure}
    \begin{subfigure}{0.49\columnwidth}
        \centering
        \includegraphics[width=\columnwidth]{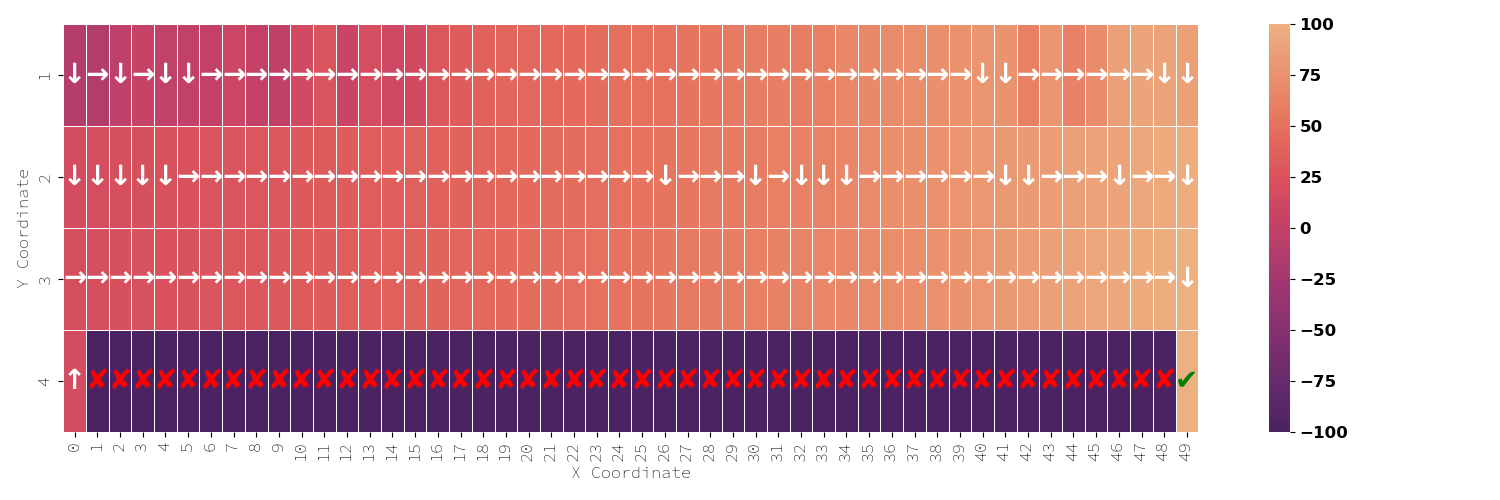}
        \caption{GRM $D = 10$, not normalized}
        \label{fig:large-grm10nonorm}
    \end{subfigure}
    \begin{subfigure}{0.49\columnwidth}
        \centering
        \includegraphics[width=\columnwidth]{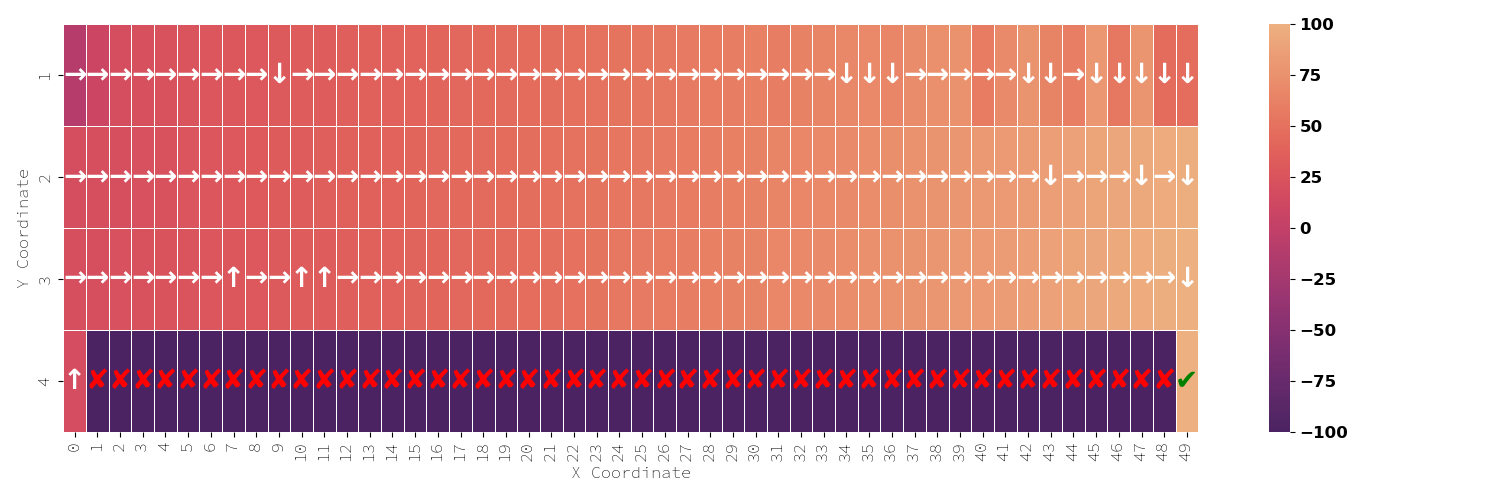}
        \caption{GRM $D = 10$}
        \label{fig:large-grm10}
    \end{subfigure}
    \caption{Final policies of trained agents and their estimated Q-values on the long cliff walking environment. Arrows indicate the action with the highest estimated Q-value in each position. A brighter hue indicates a higher Q-value.}
    \label{large-Q_values}
\end{figure}

\section{Conclusion}

We've extended PBRS to a more general class of reward functions than has been covered previously in the literature, and proven that important theoretical guarantees---namely, the preservation of the set of optimal policies for the underlying environment---still hold. We have also provided a computationally efficient and effective class of ``plug-and-play'' methods to convert most SOTA IM methods into this optimality-preserving form, proven our methods' generality, and demonstrated their efficacy at both preventing IM reward hacking and, in some circumstances, accelerating training.

\vskip 0.2in
\bibliography{main}

\end{document}